\documentclass{article}

\usepackage[T1]{fontenc}
\usepackage{hyperref}
\usepackage{microtype}

\usepackage{subcaption}
\usepackage{tikz}

\usepackage[final,nonatbib]{neurips_2021}

\usepackage[commands,environments,enumerate,citations,centerfigures,theorems]{AVT}
\let\UrlFont\sffamily

\bibliography{Vector-Field-GP}

\title{Vector-valued Gaussian Processes on Riemannian Manifolds via Gauge Independent Projected Kernels}

\author{Michael Hutchinson\textsuperscript{\ensuremath{*1}}
  \\\bfseries So Takao\textsuperscript{\ensuremath{*5}}
  \And Alexander Terenin\textsuperscript{\ensuremath{*2,3}}
  \\\bfseries Yee Whye Teh\textsuperscript{\ensuremath{1}}
  \And Viacheslav Borovitskiy\textsuperscript{\ensuremath{*4}}
  \\\bfseries Marc Peter Deisenroth\textsuperscript{\ensuremath{5}}
  \AND
  \normalfont
  \textsuperscript{\ensuremath{1}}University of Oxford
  \quad
  \textsuperscript{\ensuremath{2}}University of Cambridge
  \quad
  \textsuperscript{\ensuremath{3}}Imperial College London
  \\
  \textsuperscript{\ensuremath{4}}St. Petersburg State University
  \,\,\,
  \textsuperscript{\ensuremath{5}}Centre for Artificial Intelligence, University College London
}

\begin{document}

\maketitle

\begin{abstract}
Gaussian processes are machine learning models capable of learning unknown functions in a way that represents uncertainty, thereby facilitating construction of optimal decision-making systems.
Motivated by a desire to deploy Gaussian processes in novel areas of science, a rapidly-growing line of research has focused on constructively extending these models to handle non-Euclidean domains, including Riemannian manifolds, such as spheres and tori.
We propose techniques that generalize this class to model vector fields on Riemannian manifolds,  which are important in a number of application areas in the physical sciences.
To do so, we present a general recipe for constructing gauge independent kernels, which induce Gaussian vector fields, i.e. vector-valued Gaussian processes coherent with geometry, from scalar-valued Riemannian kernels.
We extend standard Gaussian process training methods, such as variational inference, to this setting.
This enables vector-valued Gaussian processes on Riemannian manifolds to be trained using standard methods and makes them accessible to machine learning practitioners.
\end{abstract}

\section{Introduction}

Gaussian processes are an effective model class for learning unknown functions.
They are particularly attractive for use within data-efficient decision systems, including Bayesian optimization \cite{Brochu2009, Osborne2009, Shahriari2016}, model-based reinforcement learning \cite{Rasmussen2004, Deisenroth2011c}, and active learning \cite{Krause2008}.
In these settings, Gaussian processes can represent and propagate uncertainty, as well as encode inductive biases as prior information in order to drive data efficiency. 
A key aspect of prior information is the geometry of the domain on which the Gaussian process is defined, which often encodes key properties, such as symmetry.
Following the growing deployment of Gaussian processes, a number of recent works have focused on how to define Gaussian processes on non-Euclidean domains in ways that reflect their geometric structure \cite{Borovitskiy2020, Borovitskiy2021}.
{
\renewcommand\thefootnote{}\footnotetext[0]{\textsuperscript{\ensuremath{*}}Equal contribution. Code: \url{https://github.com/MJHutchinson/ExtrinsicGaugeIndependentVectorGPs}. For a general implementation, see \url{https://github.com/GPflow/GeometricKernels/}.}
}

In many applications, such as climate science, quantities of interest are vector-valued. 
For example, global wind velocity modeling must take into account both speed and direction, and is represented by a vector field.
On geometric domains, the mathematical properties of vector fields can differ noticeably from their Euclidean counterparts: for instance, one can prove that every smooth vector field on a sphere must vanish in at least one point \cite{lee2013smooth}.
Behavior such as this simultaneously highlights the need to represent geometry correctly when modeling vector-valued data, and presents a number of non-trivial technical challenges in constructing models that are mathematically sound.

In particular, even the classical definition of a vector-valued Gaussian process---that is, a random function with multivariate Gaussian marginals at any finite set of points---already fails to be a fully satisfactory notion when considering smooth vector fields on a sphere.
This is because tangent vectors at distinct points live within different tangent spaces, and it is not clear how to construct a cross-covariance between them that does not depend on a completely arbitrary choice of basis vectors within each space.
Constructions that are independent of this choice of basis are called \emph{gauge independent}, and  recent work \cite{weiler2021coordinate,cohen2019gauge,haan2021gauge} in geometric machine learning has focused on satisfying this key property for convolutional neural networks that deal with non-Euclidean data.

Our contributions include the following.
We (a) present a differential-geometric formalism for defining Gaussian vector fields on manifolds in a coordinate-free way, suitable for Gaussian process practitioners with minimal familiarity with differential geometry, (b) present a universal and fully constructive technique for defining prior Gaussian vector fields on Riemannian manifolds, which we term the \emph{projected kernel} construction, and (c) discuss how to adapt key components in the computational Gaussian process toolkit, such as inducing point methods, to the vector field setting.

The structure of the paper is as follows. 
In Section~\ref{sec:gp}, we define vector-valued Gaussian processes on smooth manifolds.  
We start by reviewing the multi-output Gaussian process set-up, which is typically used in machine learning. 
We then detail a differential-geometric formalism for defining vector-valued Gaussian processes on smooth manifolds. 
In Section~\ref{sec:model construction}, we provide a concrete construction for these Gaussian processes on Riemannian manifolds and discuss how they can be trained using variational sparse approximations. 
Section~\ref{sec:examples} showcases Gaussian vector fields on two tasks, namely weather imputation from satellite observations and learning the dynamics of a mechanical system. 

\section{Vector-valued Gaussian Processes on Smooth Manifolds}
\label{sec:gp}

A vector-valued Gaussian process (GP) is a random function $\v{f} : X \-> \R^d$ such that, for any finite set of points $\v{x} \in X^n$, the random variable $\v{f}(\v{x}) \in \R^{n \x d}$ is jointly Gaussian.
Every such GP is characterized by its mean function $\v{\mu} : X \-> \R^d$ and matrix-valued covariance kernel $k : X \x X \-> \R^{d \x d}$, which is a positive-definite function in the matrix sense.
These functions satisfy $\E(\v{f}(\v{x})) = \v{\mu}(\v{x})$ and $\Cov(\v{f}(\v{x}), \v{f}(\v{x}^\prime)) = k(\v{x},\v{x}^\prime)$ for any $\v{x}, \v{x}^\prime\in X$.
Here, dependence between function values is encoded in the kernel's variability along its input domain, and correlations between different dimensions of the vector-valued output are encoded in the matrix that the kernel outputs.

Consider a function $\v{f}$ with $\v{y} = \v{f}(\v{x}) + \v\eps$, where $\v\eps\sim \mathrm{N}(\v{0}, \sigma^2\m{I})$ and training data $(\v{x}, \v{y})$. Placing a GP prior $\v{f}\~[GP](0,k)$ on the unknown function results in a GP posterior, whose  mean and covariance are given by 
\[ \label{eqn:gp_dist_cond}
\E(\v{f}\given\v{y}) &= \m{K}_{(\.)\v{x}} (\m{K}_{\v{x}\v{x}} + \sigma^2\m{I})^{-1}\v{y}
&
\Cov(\v{f}\given\v{y}) &= \m{K}_{(\.,\.)} -   \m{K}_{(\.)\v{x}} (\m{K}_{\v{x}\v{x}} + \sigma^2\m{I})^{-1}  \m{K}_{\v{x}(\.)}.
\]
Here, $(\.)$ denotes an arbitrary set of test locations, $\m{K}_{\v{x} \v{x}} = k(\v{x}, \v{x})$ is the kernel matrix, and $\m{K}_{ (\cdot)\v{x}} = k(\v{x}, \cdot)$ is the cross-covariance matrix between function values evaluated at the training and test inputs. 
The GP posterior can also be written as
\[ \label{eqn:gp_path_cond}
(\v{f}\given\v{y})(\.) = \v{f}(\.) + \m{K}_{(\.)\v{x}} (\m{K}_{\v{x}\v{x}} + \sigma^2\m{I})^{-1}(\v{y} - \v{f}(\v{x}) - \v\eps),
&
&
\v\eps &\~[N](\v{0},\sigma^2\m{I})
\]
where $\v{f}(\.)$ is the prior GP, and equality holds in distribution \cite{wilson2020pathwise,wilson2021pathwise}.
These expressions form the foundation upon which Gaussian-process-based methods in machine learning are built.

Recent works have studied techniques for working with the expressions~\eqref{eqn:gp_dist_cond}~and~\eqref{eqn:gp_path_cond} when the input domain $X$ is a Riemannian manifold, focusing both on defining general classes of kernels \cite{Borovitskiy2020}, and on efficient computational techniques \cite{wilson2020pathwise,wilson2021pathwise}.
In this setting, namely for $f : X \-> \R$, defining kernels already presents technical challenges: the seemingly-obvious first choice one might consider, namely the geodesic squared exponential kernel, is ill-defined in general \cite{feragen2015geodesic}.
We build on these recent developments to model vector fields on manifolds using GPs.
We do not consider manifold-valued generalizations of Gaussian processes, for instance $f : \R \-> X$: various constructions in this setting are instead studied by \textcite{stroock2000introduction, emery2012stochastic,mallasto2018wrapped,mallasto2019probabilistic}.
To begin, we review what a vector field on a manifold actually is.

\subsection{Vector Fields on Manifolds}

Let $X$ be a $d$-dimensional smooth manifold with $T_x X$ denoting its tangent space at $x$. Let ${TX = \cbr{(x, v) \,|\, x \in X, v \in T_xX}}$ be its \emph{tangent bundle}, and let ${T^* X = \cbr{(x, \phi) \,|\, x \in X, \phi \in T_x^*X}}$ be its \emph{cotangent bundle}---endow both spaces with the structure of smooth manifolds.
Define the projection map $\proj_X : TX \-> X$ by $\proj_X (x, v) = x$.
A \emph{vector field} on $X$ is a map that assigns each point in $X$ to a tangent vector attached to that point.
More formally, a vector field is a \emph{cross-section}, or simply a \emph{section}, of the tangent bundle, which is a map $f : X \-> TX$, such that $\proj_X f(x) = x$ for all $x$.\footnote{A vector field is \emph{not the same as} a map $\tilde{f} : X \to \R^d$: an output value $f(x) \in TX$ formally consists of both a copy of the input point $x$, and vector within the tangent space $T_x X$ at this point. This encodes the geometric structure of the underlying manifold. The algebraic requirement $\proj_X f(x) = x$ for all $x$ ensures that the tangent vector chosen correctly corresponds to the point at which it is attached.} A vector field is called \emph{smooth} if this map $f$ is smooth.

To represent a vector field on a manifold numerically, one must choose a basis in each tangent space, which serves as a coordinate system for vectors in the tangent space.
On many manifolds it is \emph{impossible} to choose these basis vectors in a way that they vary smoothly in space.\footnote{If a smooth choice of basis vectors existed, it would define a smooth non-vanishing vector field. On the sphere, by the \emph{hairy ball theorem}, all smooth vector fields vanish in at least one point, so no such bases exist.}
This can be handled by working with local coordinates, or with bases that are non-smooth.
Any chosen set of basis vectors is arbitrary, so objects constructed using them should not depend on this choice.
Constructions that satisfy this property are called \emph{gauge independent}.
This notion is illustrated in Figure \ref{fig:gauges}, and will play a key role in the sequel.

\begin{figure}
\begin{tikzpicture}
\node at (0,2.3375) {\includegraphics[scale=0.25]{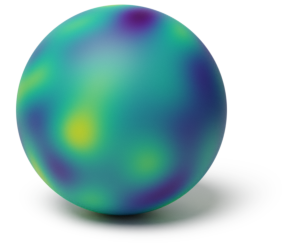}};
\node at (0,0) {\includegraphics[scale=0.25]{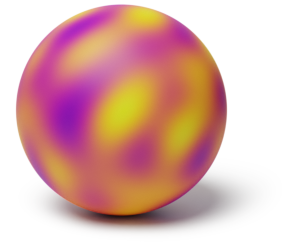}};
\node at (2.26875,2.3375) {\includegraphics[scale=0.25]{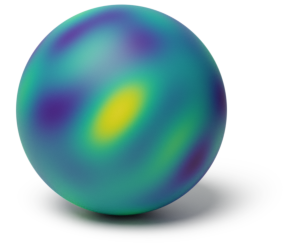}};
\node at (2.26875,0) {\includegraphics[scale=0.25]{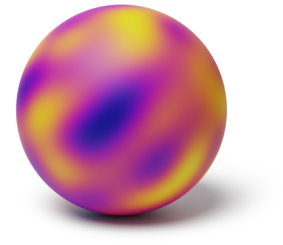}};
\node at (3.678125,2.3375) {\LARGE\ensuremath{\x}};
\node at (3.678125,0) {\LARGE\ensuremath{\x}};
\node at (5.3625,2.3375) {\includegraphics[scale=0.25]{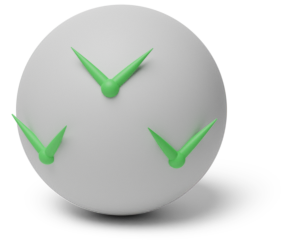}};
\node at (5.3625,0) {\includegraphics[scale=0.25]{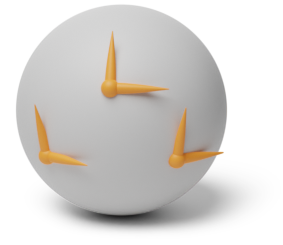}};
\node at (6.7375,1.16875) {\Huge\ensuremath{=}};
\node at (9.7625,1.1) {\includegraphics[scale=0.25]{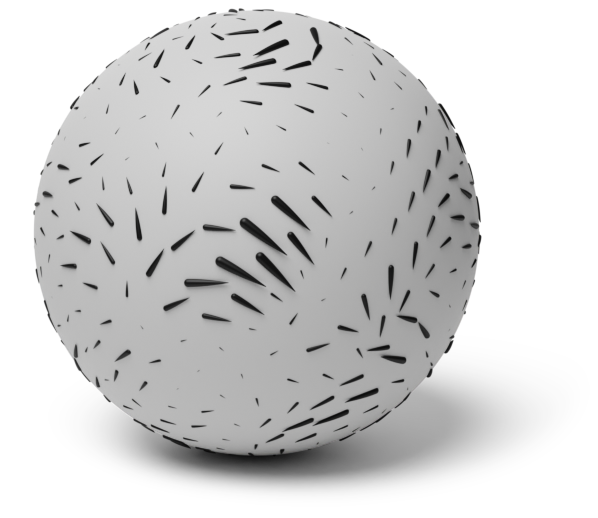}};
\node at (0.9625,3.9875) {Basis coefficients};
\node at (5.190625,3.9875) {Frame};
\node at (9.384375,3.9875) {Vector field};
\end{tikzpicture}
\caption{Illustration on $\mathbb{S}^2$. Here we illustrate two possible bases (also called frames), consisting of green and orange basis vectors (center), that can be chosen locally on the manifold $\mathbb{S}^2$. The vector field on $\mathbb{S}^2$ (right) can be produced by taking two scalar fields (left) in each respective color, and combining them with the basis vectors (center) to form the vector field.}
\label{fig:gauges}
\end{figure}
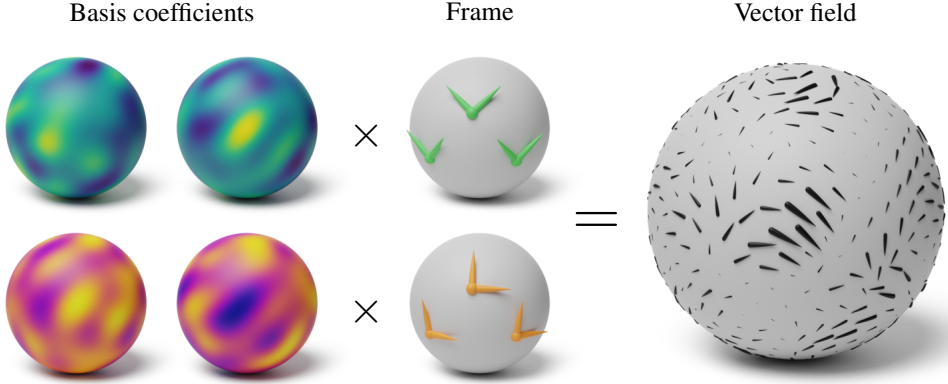
\subsection{Gaussian Vector Fields}
\label{sec:vv-manifold-gp}

Upon reflecting on the above considerations in the context of GPs, the first issue one encounters is that, for a random vector field $f : X \-> TX$, it is not clear what it means for finite-dimensional marginal distributions to be \emph{multivariate Gaussian} given that $f$ takes its values in a bundle rather than a vector space.
The first step towards constructing Gaussian vector fields, therefore, involves adapting the notion of \emph{finite-dimensional marginal distributions} appropriately.

\begin{definition}
\label{def:vect-gp}
Let $X$ be a smooth manifold.
We say that a random vector field $f$ is \emph{Gaussian} if for any finite set of locations $x_1, \ldots, x_n \in X$, the random vector $(f(x_1),\ldots,f(x_n)) \in T_{x_1}X \oplus \ldots \oplus T_{x_n}X$ is Gaussian (either in the sense of duality or in any basis: see Appendix A for details).
\end{definition}

This definition is near-identical to the Euclidean case: the only difference is that finite-dimensional marginals are now supported in a direct sum of tangent spaces, instead of $\R^{n \x d}$.
With this definition, the standard multi-output GP properties, such as conditioning, carry over, virtually unmodified.
Definition \ref{def:vect-gp} is a natural choice: if we embed our manifold into Euclidean space, the induced GP is a vector-valued GP as defined in the beginning of Section~\ref{sec:gp}.

\begin{proposition} \label{prop:embedding} Let $X$ be a manifold and $\f{emb} : X \-> \R^p$ be a smooth embedding.
Let $f$ be a Gaussian vector field (as defined in Definition~\ref{def:vect-gp}), and let $f_{\f{emb}} : \f{emb}(X) \-> \R^p$ be its pushforward along the embedding.
Then $f_{\f{emb}}$ is a vector-valued Gaussian process in the Euclidean sense.
\end{proposition}

All proofs in this work can be found in Appendix \ref{apdx:theory}.
Having established the notion of a vector-valued Gaussian process on a smooth manifold, we proceed to deduce what mathematical objects play the role of a mean function and kernel, so that it is clear what ingredients are needed to construct and determine such a process.

The former is clear: the mean of a Gaussian vector field should be an ordinary vector field, and will determine the mean vector at all finite-dimensional marginals.
The kernel, on the other hand, is less obvious: because distinct tangent vectors live in different tanget spaces, it is unclear whether or not a Gaussian vector field is characterized by an appropriate notion of a matrix-valued kernel, or by something else. Now, we define the right notion for kernel in this setting.
\begin{definition}
We call a symmetric function $k : T^*X \x T^*X \-> \R$ \emph{fiberwise bilinear} if for all pairs of points $x, x' \in X$
\[
k(\lambda \alpha_x + \mu \beta_x, \gamma_{x'}) &= \lambda k(\alpha_x, \gamma_{x'}) + \mu k(\beta_x, \gamma_{x'})
\]
holds for any $\alpha_x, \beta_x \in T^*_x X$, $\gamma_{x'} \in T^*_{x'} X$ and $\lambda, \mu \in \R$, and \emph{positive semi-definite} if for any set of covectors $\alpha_{x_1}, \ldots, \alpha_{x_n} \in T^*X$, we have $\sum_{i=1}^n\sum_{j=1}^n k(\alpha_{x_i}, \alpha_{x_j}) \geq 0$.
We call a symmetric fiberwise bilinear positive semi-definite function a \emph{cross-covariance kernel}.
\end{definition}

This coordinate-free function should be viewed as analogous to $((x,\v{v}),(x',\v{v}')) \|> \v{v}^T\m{K}_{x,x'}\v{v}'$ in the Euclidean setting, where $\v{v},\v{v}'$ multiply the matrix-valued kernel from both sides.
Its coordinate representation, which more closely matches the Euclidean case, will be explored in the sequel.
To show that this is indeed the right notion, we prove the following result.

\begin{theorem}
\label{prop:riem-kern}
The system of marginal distributions of a Gaussian vector field on a smooth manifold $X$ is uniquely determined by a mean vector field $\mu: X \-> TX$ and a cross-covariance kernel ${k : T^*X \x T^*X \-> \R}$. Moreover, this correspondence is one-to-one.
\end{theorem}

By virtue of defining and characterizing all Gaussian vector fields, Theorem~\ref{prop:riem-kern} assures us the definition of a kernel introduced is the correct mathematical notion.
The constructions presented here are all intrinsic or, in other words, coordinate-free, and do not involve the use of bases.
To understand how to perform numerical calculations with these kernels we proceed to study their coordinate representations with respect to a specific choice of basis.

\subsection{Matrix-valued Kernels}
\label{sec:gauge-equivariant-kernels}

In Section~\ref{sec:vv-manifold-gp}, we defined what a Gaussian vector field on a manifold is.
However, by nature of the manifold setting, the resulting objects are more abstract than usual and do not describe how it can be represented numerically.
We now develop a point of view suitable for this task.

To this end, we introduce a \emph{frame} $F$ on $X$, also known as a \emph{gauge} in physical literature, which is a collection of (not necessarily smooth) vector fields $e_1, \ldots, e_d$ on $X$ such that at each point $x \in X$, the set of vectors $e_1(x), \ldots, e_d(x)$ forms a basis of $T_xX$. 
The frame allows us to express a vector field $f$ on $X$ as simply a vector-valued function $\v{f} = (f^1, \ldots, f^d) : X \-> \R^d$, such that $f(x) = \sum_{i=1}^d f^i(x) e_i(x)$ for all $x \in X$.
The corresponding \emph{coframe} $F^*$ is defined as a collection $e^1, \ldots, e^d$ of covector fields (one-forms) on $X$ such that $\left<e^i(x) | e_j(x)\right> = \delta_{ij}$ for all $x \in X$, where $\delta_{ij}$ is the Kronecker delta.
In the following proposition, we show that if $f$ is a Gaussian vector field on $X$ (in the sense of Definition \ref{def:vect-gp}), then the corresponding vector representation~$\v{f}$ expressed in a given frame is a vector-valued GP in the standard sense.

\begin{proposition}\label{prop:matrix-representation}
Let $f$ be a Gaussian vector field defined on $X$ with cross-covariance kernel ${k : T^*X \x T^*X \-> \R}$. Given a frame $F = (e_1, \ldots, e_d)$ on $X$, define $\v{f} : X \-> \R^d$ as above. Then $\v{f}$ is a vector-valued GP in the usual sense with kernel ${\m{K}_F : X \x X \-> \R^{d \x d}}$ given by
\[
\m{K}_F(x, x') =
\begin{bmatrix} \label{eq:K_F}
k(e^1(x), e^1(x')) & \dots  & k(e^1(x), e^d(x')) \\
\vdots           & \ddots & \vdots \\
k(e^d(x), e^1(x')) & \dots  & k(e^d(x), e^d(x'))
\end{bmatrix},
\]
where $(e^i)$, with raised indices, is the coframe corresponding to $(e_i)$.
Conversely, given a vector-valued GP $\v{f} = (f^1, \ldots, f^d): X \-> \R^d$ and a frame $F = (e_1, \ldots, e_d)$ on $X$, $f(\cdot) := \sum_{i=1}^d f^i(\cdot) e_i(\cdot)$ defines a Gaussian vector field on $X$.
\end{proposition}

This result shows precisely how numerical representations of a Gaussian vector field depends on the choice of frame. 
While this representation is not \emph{invariant} under this choice, it is \emph{equivariant}, meaning that a transformation in the frame results in an appropriate transformation of the kernel.
To make this notion precise, we introduce a matrix subgroup $G = \f{GL}(d, \R)$, called the \emph{gauge group}, that acts on $\R^d$ by a standard matrix-vector multiplication. 
Given two frames $F, F'$ on $X$, an abstract vector $f_x \in T_xX$ has two vector representations $\v{f}_x, \v{f}'_x$ in the respective frames. We say that $F'$ is obtained from $F$ by a \emph{gauge transformation} with respect to a matrix field $\m{A}: X \-> G \subseteq \R^{d \x d}$, if
\[
\v{f}'_x = \m{A}(x) \v{f}_x
\]
holds for all $x \in X$, and we write $F' = \m{A}F$. Note that $\m{A}(x)$ need not be smooth in $x$.
We see that the gauge transformation is therefore just a linear change of basis of the frame $F$ at each point for which one can identify vectors in $T_xX$ as elements in $\R^d$. 
The corresponding gauge dependant matrix-valued kernels must also respect this transformation rule, a statement of \emph{gauge independence}.

\begin{corollary}\label{cor:equivariance}
Let $F$ be a frame on $X$ and $\m{K}_F : X \x X \-> \R^{d\x d}$ be the corresponding matrix representation \eqref{eq:K_F} of a cross-covariance kernel $k: T^*X \x T^*X \-> \R$. 
This satisfies the \emph{equivariance condition}
\[
\m{K}_{\m{A} F}(x, x') = \m{A}(x) \m{K}_{F}(x,x') \m{A}(x')^T,
\]
where $\m{A} : X \-> G \subseteq \R^{d \x d}$ is a gauge transformation.
All cross-covariance kernels in the sense of Proposition~\ref{prop:riem-kern} arise this way.
\end{corollary}

Hence, one way to define a Gaussian vector field on a manifold is to find a gauge independent kernel.
In summary, we have described Gaussian vector fields in a coordinate-free differential-geometric language, and deduced enough properties to confirm the objects defined truly deserve to be called GPs.
In doing so, we have both introduced the necessary formalism to the GP community, and obtained a recipe for defining kernels, through a simple condition atop standard matrix-valued kernels.
To proceed towards practical machine learning methods, we therefore study techniques for constructing such kernels explicitly.

\section{Model Construction and Bayesian Learning for Riemannian Manifolds}
\label{sec:model construction}

In Section~\ref{sec:gp}, we introduced a notion of a Gaussian vector field.
We now study how to use vector fields for machine learning purposes.
This entails two primary issues: (a) how to construct practically useful kernels, and (b) once a kernel is constructed, how to train Gaussian processes.

To construct a Gaussian vector field prior, the preceding theory tells us that we need to specify a mean vector field and  a cross-covariance kernel.
From the definition, it is not at all obvious how to specify a natural kernel, and experience with the scalar-valued case---where the innocuous-looking geodesic squared exponential kernel is generally not positive semi-definite on most Riemannian manifolds \cite{feragen2015geodesic}---suggests that the problem is delicate, i.e., simply guessing the kernel's form is unlikely to succeed.
Our goal, therefore, is to introduce a general construction for building wide classes of kernels from simple building blocks.

The same issues are present if we consider variational approximations to posterior GPs, such as the inducing point framework of \textcite{Titsias2009}: these are formulated using matrix-vector expressions involving kernel matrices, and it is important for the approximate posterior covariance to be gauge independent in order to lead to a valid approximate process.
We proceed to address these issues.

\subsection{Projected Kernels}

Here, we introduce a general technique for defining cross-covariance kernels $k : T^*X \x T^*X \-> \R$ and for working with such functions numerically.
Section~\ref{sec:gp} gives us a promising strategy to construct a suitable kernel---namely, it suffices to find a \emph{gauge independent matrix-valued kernel}.
At first glance, it is not obvious how to construct such a kernel in the manifold setting.
On many manifolds, such as the sphere, owing to the hairy ball theorem, every frame must be discontinuous: therefore, constructing a continuous kernel in such a choice of frame appears difficult.

To both get around these obstacles, and aid numerical implementation, we propose to isometrically embed the manifold into Euclidean space.\footnote{An embedding $\f{emb}: X \to \R^{d'}$ is called \emph{isometric} if it preserves the metric tensor. By Nash's Theorem \cite{lee2013smooth}, such embeddings exist for any $d$-dimensional manifold, with an embedded dimension $d' \leq 2d+1$.}
Doing so greatly simplifies these issues by virtue of making it possible to represent the manifold using a single global coordinate system.
On the other hand, the main trade-off from this choice is that by its extrinsic nature, the construction can make theoretical analysis more difficult.
To proceed, we need two ingredients. 
\1 An isometric embedding $\f{emb}: X \-> \R^{d'}$ of the manifold.
\2 A vector-valued Gaussian process $\v{f}': X \-> \R^{d'}$ in the standard sense.\footnote{We emphasize again that $\v{f}'$ is \emph{not} a Gaussian vector field because it is not a random section. In particular, note that $d' > d$ for most embeddings.}
\0
A simple choice which reflects the geometry of the manifold is to take $\v{f}'$ to be $d'$ independent scalar-valued GPs on $X$.

By standard results in differential geometry, any smooth map $\phi : X \-> X'$ between two manifolds induces a corresponding linear map on the tangent spaces $\d_x\phi : T_xX \-> T_{\phi(x)}X'$, which can loosely be thought of as mapping $\phi$ to its first-order Taylor expansion at $x$. Thus, an embedding $\f{emb}: X \-> \R^{d'}$, induces a map $\d_x\!\f{emb}: T_xX \-> T_{\f{emb}(x)}\R^{d'}$. Now fixing a frame $F$ on $X$, each tangent space $T_xX$ can be identified with $\R^d$, so without loss of generality, the map $\d_x\!\f{emb}$ can be expressed simply as a position-dependent matrix $\m{P}_x^T \in \R^{d' \x d}$. Taking the transpose, we obtain $\m{P}_x \in \R^{d \x d'}$, which we call the \emph{projection matrix}.
The desired Gaussian vector field on $X$, with respect to $F$, is then constructed as $\v{f}(x) = \m{P}_x \v{f}'(x)$. This procedure is illustrated in Figure \ref{fig:kernel-embedding}: there we see that to get a vector field on an $\R^3$-embedded sphere, we may take a vector-valued function on it and project its values to make vectors tangential to this sphere, thus obtaining a valid vector field.
Since the projection operator preserves smoothness and since we can take a smooth vector-valued GP to begin with, it is clear that this approach may be used to build smooth vector fields.

\begin{figure}

\begin{tikzpicture}
\node at (1.1,1.1) {\includegraphics[scale=0.25]{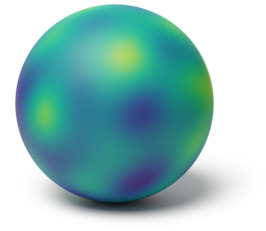}};
\node at (0,-0.9625) {\includegraphics[scale=0.25]{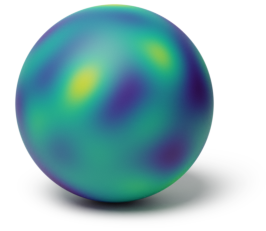}};
\node at (2.2,-0.9625) {\includegraphics[scale=0.25]{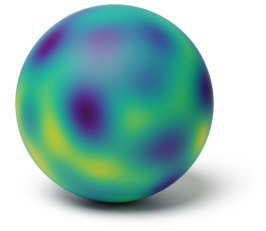}};
\node at (5.6375,0) {\includegraphics[scale=0.25]{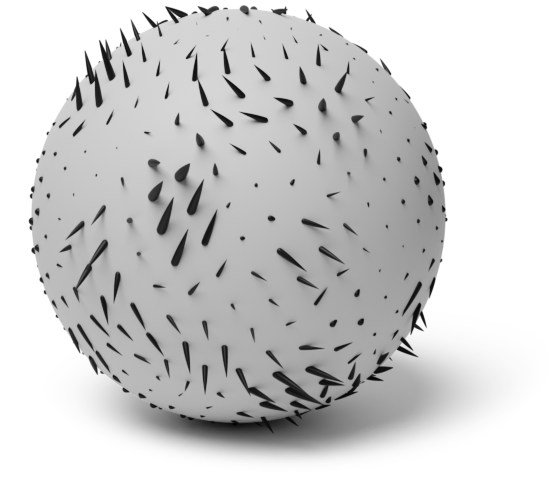}};
\node at (10.0375,0) {\includegraphics[scale=0.25]{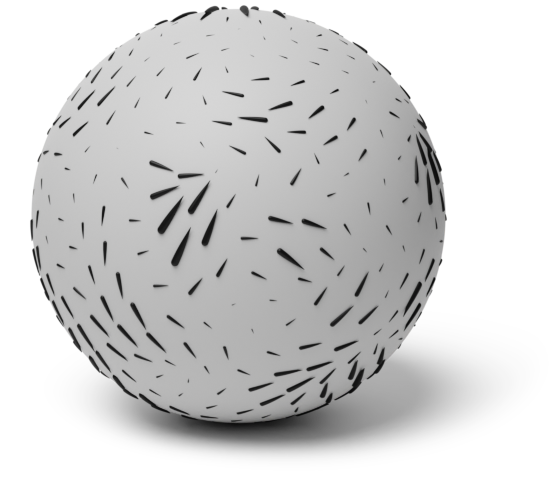}};
\node at (0.928125,2.646875) {Scalar processes};
\node at (5.29375,2.646875) {Embedded process};
\node at (9.69375,2.646875) {Projected process};
\end{tikzpicture}
\caption{Illustration of the construction process of the projected process. The manifold $\mathbb{S}^2$ is embedded into $\R^3$. Three identical scalar GPs (left) are placed on the manifold. These three scalar GPs are combined to construct a vector-valued GP in the ambient Euclidean space (center). This GP is then projected onto the tangent space of $\mathbb{S}^2$ as a subspace of the tangent space of $\R^3$ (right).}
\label{fig:kernel-embedding}
\end{figure}
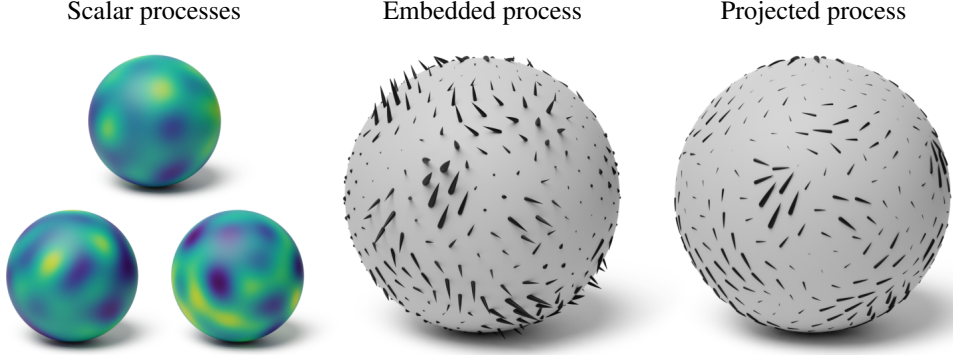

We prove that (a) the resulting expression is, indeed, a kernel, and that (b) no expressivity is lost via the construction because all cross-covariance kernels arise this way.

\begin{proposition} \label{prop:projected-kernel}
Let $(X, g)$ be a Riemannian manifold, $\f{emb} : X \-> \R^{d'}$ be an isometric embedding and $F$ be a frame on $X$.
We denote by $\m{P}_{(\.)} : X \-> \R^{d \x d'}$ the associated projection matrix under $F$, and let $\v{f}' : X \-> \R^{d'}$ be any vector-valued Gaussian process with matrix-valued kernel $\v\kappa : X \x X \-> \R^{d' \x d'}$.
Then, the vector-valued function $\v{f} = \m{P}\v{f}'$ defines a Gaussian vector field $f$ on $X$ using the construction in Proposition \ref{prop:matrix-representation}, whose kernel under the frame $F$ has matrix representation
\[\label{eq:projection-relation}
\m{K}_F(x, x') = \m{P}_x \v\kappa(x,x') \m{P}_{x'}^T.
\]
Moreover, all cross-covariance kernels $k : T^*X \x T^*X \-> \R$ arise this way.
We call a kernel defined this way a \emph{projected kernel}.
\end{proposition}

To construct these kernels we require scalar-valued kernels on manifolds to use as a basic building block.
These are studied in the general Riemannian setting by \textcite{lindgren11} and \textcite{Borovitskiy2020}: relying on these kernels is the only reason we require the Riemannian structure.
It is also possible to obtain such kernels using embeddings, following \textcite{lin2019extrinsic}.
Similar techniques to those we consider are used by \textcite{freeden2008spherical} to construct vector-valued zonal kernels on the sphere: in contrast, we work with arbitrary manifolds.
The projection kernel idea is a very general way to build kernels for vector fields by combining scalar kernels, but effective scalar kernels, naturally, rely on Riemannian structure.\footnote{The structure of a smooth manifold is not rigid enough to define natural kernels. For instance, the smooth structure of the sphere is indistinguishable from the smooth structure of an ellipsoid or even of the dragon manifold from~\textcite{Borovitskiy2020}, but their Riemannian structures differ considerably.}
\Cref{fig:example_samples} shows random samples from Gaussian processes constructed with the described kernels.

The projected kernel construction both makes it easy to define cross-covariance kernels on general manifolds, and describes a straightforward way to implement them numerically by representing the embedded manifold in coordinates and calculating the resulting matrix-vector expressions.
The constructed kernel depends on the embedding, but can be transformed appropriately if switching to a different embedding.
Embeddings, in turn, are available for most manifolds of practical interest, and are obtained automatically for manifolds approximated numerically as meshes. 
Everything described is constructive and fully compatible with the modern automatic-differentiation-based machine learning toolkit, and most operations for constructing and/or sampling from specialized priors \cite{vanderwilk20,lange2018algorithmic, lange2021linearly}, including on spaces such as the sphere where specific analytic tools are available \cite{creasey2018fast, dutordoir2020sparse, emery2019simulating, emery2019turning}.
With these kernels in hand, we thus proceed to study training methods.

\begin{figure}
\begin{tikzpicture}
\node at (-3.5,0) {};
\node[anchor=south] at (0,0) {\includegraphics[scale=0.25]{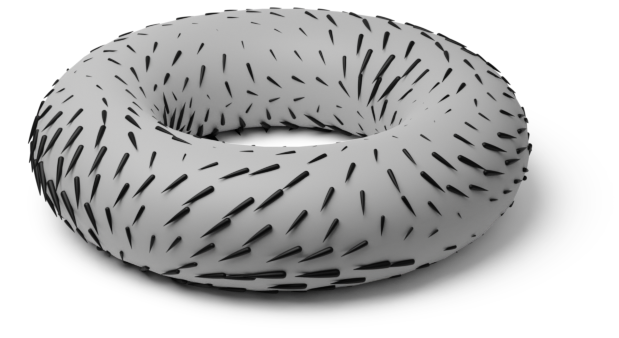}};
\node[anchor=south] at (6,0) {\includegraphics[scale=0.25]{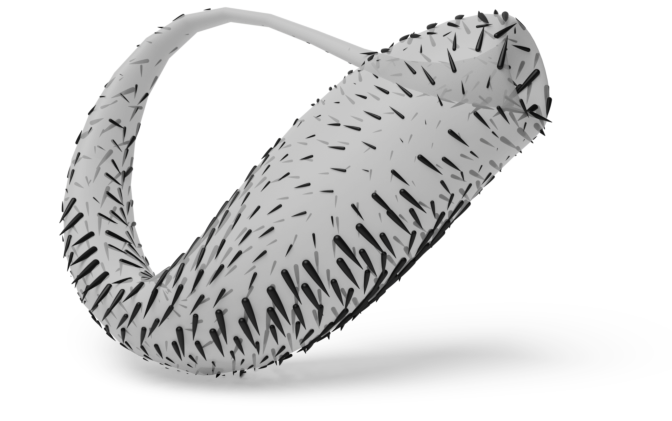}};
\end{tikzpicture}
\caption{Random samples from Gaussian processes with gauge independent projected kernels, on the torus and Klein bottle, respectively. The latter is a non-orientable manifold: the ability to handle such cases highlights the generality of projected kernels.}
\label{fig:example_samples}
\end{figure}
\subsection{Gauge Independent Variational Approximations}

We now discuss variational inference for training GPs in the Riemannian vector field setting.
Approximations, such as the inducing-point framework by \textcite{Titsias2009} and \textcite{hensman13}, approximate the posterior GP with another GP, termed the \emph{variational approximation}.
The latter is typically constructed by specifying a multivariate Gaussian at a set of test locations with a parameterized mean and kernel matrix. For example, \textcite{opper2009variational} consider $\f{N}(\m{m}, \m{S})$, where
\[ \label{eqn:opper_variational}
\m{m} = \m{K}_{(\.)\v{z}}(\m{K}_{\v{z}\v{z}} + \m\Sigma)^{-1} \v\mu
\qquad
\m{S} =  \m{K}_{(\.,\.)} - \m{K}_{(\.)\v{z}}(\m{K}_{\v{z}\v{z}} + \m\Sigma)^{-1} \m{K}_{\v{z}\v{z}} (\m{K}_{\v{z}\v{z}} + \m\Sigma)^{-1} \m{K}_{\v{z}(\.)}.
\]
The variational parameters include a set of inducing locations $\v{z}$, a mean vector $\v\mu$, and a block-diagonal cross-covariance matrix $\m\Sigma$.
Training proceeds by optimizing these parameters to minimize the Kullback--Leibler divergence of the variational distribution from the true posterior, typically using mini-batch stochastic gradient descent.

In the last decade, a wide and diverse range of inducing point approximations suited for many different settings have been proposed \cite{Titsias2009,opper2009variational,lazaro2009inter,vanderwilk20,wu2021hierarchical}.
The vast majority of them employ coordinate-dependent matrix-vector expressions.
This raises the question, which of these constructions can be adapted to define valid variational approximations in the vector field setting?

To proceed, one can choose a frame and formulate a given variational approximation using matrices defined with respect to this frame.
To ensure well-definedness, one must ensure that all these matrices, such as the kernel matrix and the variational parameter $\m\Sigma$ in \eqref{eqn:opper_variational}, are gauge independent.
These considerations can be simplified adopting the pathwise view of GPs, and examining the random variables directly.
For example, the variational approximation of \textcite{opper2009variational} shown previously in \eqref{eqn:opper_variational} can be reinterpreted pathwise as the GP
\[
(f\given y)(\.) &\approx f(\.) + \m{K}_{(\.)\v{z}}(\m{K}_{\v{z}\v{z}} + \m\Sigma)^{-1} (\v\mu - f(\v{z}) - \v\eps)
&
\v\eps &\~[N](\v{0},\m\Sigma)
\]
where we view the matrices $\m{K}_{(\.)\v{z}}, \m{K}_{\v{z}\v{z}}, \m\Sigma$ as \emph{linear operators} between direct sums of tangent spaces: $\m{K}_{(\.)\v{z}} : T_{z_1}X \oplus \ldots \oplus T_{z_m}X \-> T_{(\.)}X$ and $\m{K}_{\v{z}\v{z}}, \m\Sigma : T_{z_1}X \oplus \ldots \oplus T_{z_m}X \-> T_{z_1}X \oplus \ldots \oplus T_{z_m}X$.
By virtue of being defined at the level of vector fields using components that are all intrinsically valid, the posterior covariance of the resulting variational approximation is \emph{automatically} gauge independent.
Hence, checking gauge independence is then equivalent to deducing the domains and ranges of these operators from their coordinate representations, and checking if they are compatible.
This applies to any variational family that can be constructed in the given manner.

The vast majority of inducing point constructions can be interpreted in this manner and thus extend readily to the Riemannian vector field setting by simply representing the necessary matrices in a chosen frame.
In particular, the classical approach of \textcite{Titsias2009} is gauge independent.

\section{Illustrated Examples}
\label{sec:examples}

Here, we showcase a number of examples that illustrate potential use cases of the models developed.

\subsection{Dynamical Systems Modeling}

\begin{figure}
\centering
\includegraphics{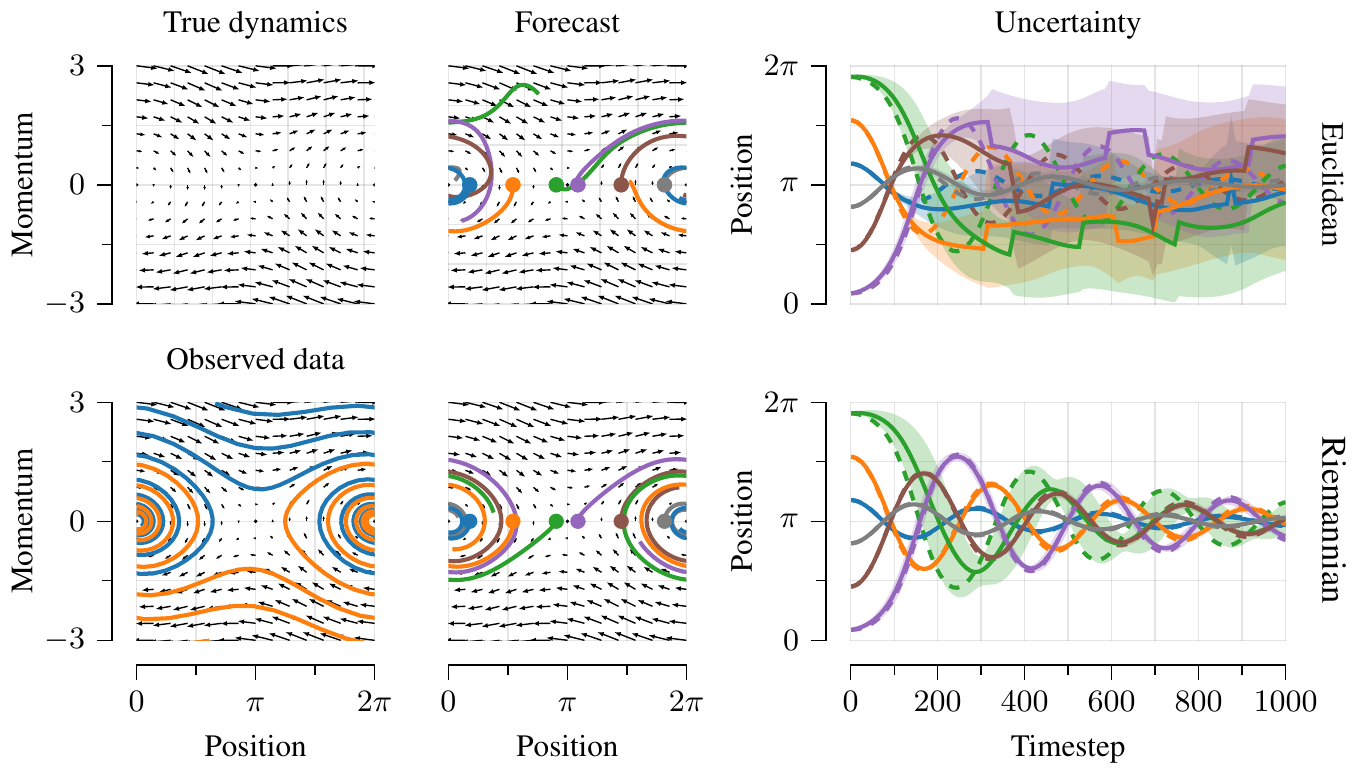}
\caption{Upper left: pendulum with friction state space. Lower left: Two rollouts used to train the GP. Upper middle: State space plot of rollouts from a standard Euclidean GP condtioned on the training data. Upper right: temporal plot of rollouts from a standard Euclidean vector GP. Solid line is the true rollout, dashed line and shade is the mean and $\pm$ 1 std of the GP rollouts. Lower middle and right: same for a geometric manifold vector field kernel on $\mathbb{S}^1\x\R$.}
\label{fig:dynamics}
\end{figure}

Here, we show how Gaussian vector fields can be used to learn the equations of motion of a physical system---an important task in imitation learning, model-based reinforcement learning, and robotics.
GPs are an attractive model class in this area owing to their ability to represent and propagate uncertainty, which enables them to separate what is known about an environment from what is not, thereby driving data-efficient exploration.

For a prototype physical system, we consider an ideal pendulum, whose configuration space is the circle $\bb{S}^1$, representing the angle of the pendulum, with zero being at the bottom of the loop, and whose position-momentum state-space is the cylinder $\bb{S}^1\x \R$.
We consider conservative dynamics with additional friction applied at the pivot. Since this system is non-conservative, we cannot just learn the Hamiltonian of the system, but must learn the vector field over the state space that defines the dynamics of the system. The true dynamics of the system are given by the differential equations
\begin{align}
\mathcal{H} &= \frac{p^2}{2ml^2} + mgl(1 - \cos(q)) 
&
\od{q}{t} &= \pd{\mathcal{H}}{p} 
&
\od{p}{t} &= -\pd{\mathcal{H}}{q} - \frac{b}{m} p,
\end{align} 
where $\mathcal{H}$ is the Hamiltonian of the system defining the conservative part of the dynamics, $q$ and $p$ are the position and momentum of the pendulum, $m$ is the mass, $l$ is the length, $g$ is the gravitational field strength and $b$ is a friction parameter.
Experimental details can be found in Appendix \ref{apdx:experiments}.

To learn this model, we initialise the system at two start points, and evolve the system using leapfrog integration. From these observations of position, we backward Euler integrate the momentum of the system, and from these position-momentum trajectories we estimate observations of the dynamics field. Using these observations, we condition a sparse GP. The result is an estimate of the system dynamics with suitable uncertainty estimates. In order to compute rollouts of these dynamics, we use pathwise sampling of this sparse GP \cite{wilson2020pathwise,wilson2021pathwise} for speed together with leapfrog integration.

Results can be seen in Figure \ref{fig:dynamics}.
While the Euclidean GP performs reasonably well at the start of the rollouts, once the trajectory crosses the discontinuity caused by looping the angle back around to zero, the system starts to make incoherent predictions: this is due to the discontinuity arising from wrap-around condition of the angle. 
The manifold vector-valued GP does not have this issue as the learned and sampled dynamics fields are continuous throughout the state-space.

\subsection{Weather Modeling}\label{sec:weather-modeling}

\begin{figure}
\begin{tikzpicture}
\node at (0,0) {\includegraphics[scale=0.25]{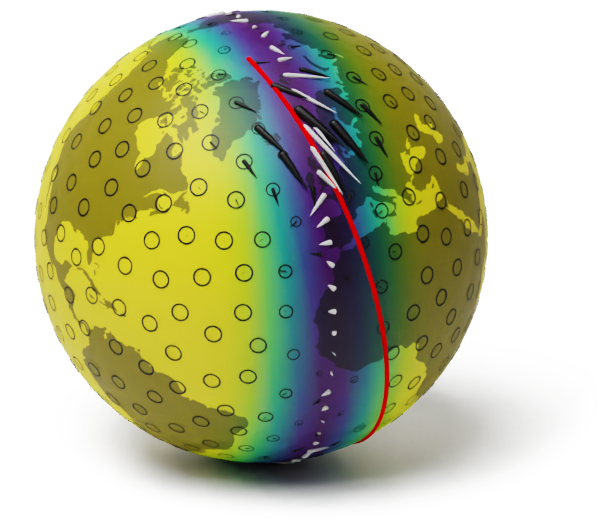}};
\node at (6.5,0.12) {\includegraphics[scale=0.25]{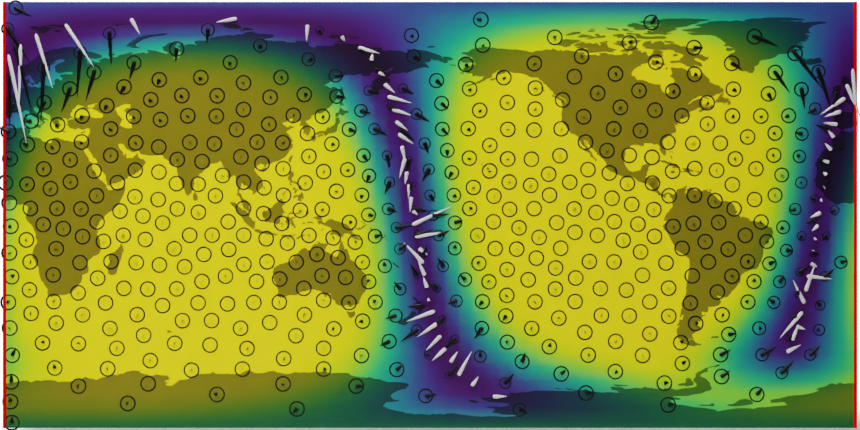}};
\node at (0,4.5) {\includegraphics[scale=0.25]{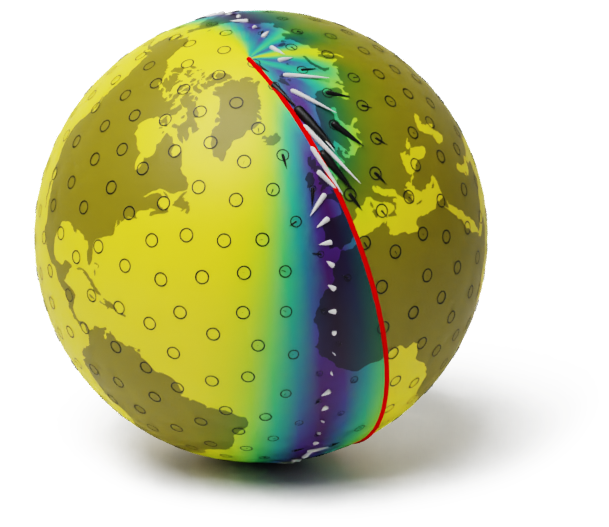}};
\node at (6.5,4.62) {\includegraphics[scale=0.25]{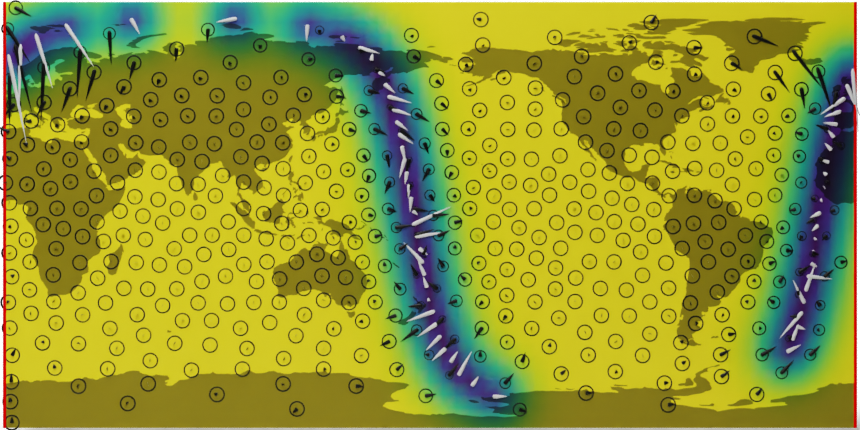}};
\end{tikzpicture}
\caption{Top row: Euclidean GP trained on wind measurements along the chosen Aeolus satellite trajectory, viewed as deviation from normal with respect to the historical average vector field. White arrows are the satellite measurements, black arrows and ellipsoids are the posterior mean and cross-covariance of the vector field, colors indicate the posterior standard deviation norm, and the solid red line indicates the latitudinal boundary when the sphere is projected onto the plane using the lat/lon projection. Bottom row: Same as above except using a manifold kernel on $\mathbb{S}^2$. 
}
\label{fig:aeolus}
\end{figure}

In this experiment, we show how vector-valued GPs on manifolds can be used in the context of meteorology, where geometric information often plays an important role in accurately modeling global weather fields \cite{weyn2020improving, scher2020spherical, llorens2020}. 
\emph{Data assimilation} in numerical weather forecasting refers to the practice of using observed data to update predictions of the atmosphere closer to the truth.
Uncertainty plays a critical role here: it is not usually possible to observe the weather at all locations on the globe simultaneously, and taking into account observation uncertainty is crucial in numerical weather forecasting during the data assimilation step \cite{law2015data, reich2015probabilistic}.
In this section, we explore Gaussian processes as a tool for carrying out global interpolation of wind fields, while simultaneously performing uncertainty quantification, mirroring \emph{optimal interpolation} techniques in data assimilation \cite{kalnay2003atmospheric}.

We consider a simplified setting, where the goal is to interpolate the wind velocity observed by the Aeolus satellite \cite{reitebuch2012spaceborne}, which uses LiDAR sensors to measure wind velocity directly.
To mimic this setting, we use an hour of the Aeolus satellite track during the period 2019/01/01 09:00-10:00 for the input locations and the wind velocity data (10m above ground) from the ERA5 atmospheric reanalysis data \cite{hersbach2016era5} interpolated at these locations, to simulate measurements taken from the Aeolus satellite.
We subtract the weekly historical average wind velocity from the observations, before training the GP models, where the historical mean is computed from the hourly wind data ($10$\,m above ground) from the WeatherBench dataset \cite{rasp2020weatherbench}, available from 1979--2018.
Further details can be found in Appendix \ref{apdx:experiments}.
We compare the results of a Mat\'ern-3/2 manifold vector-valued GP regression model fitted on the wind anomaly observations along the Aeolus trajectory, with the results from a Euclidean Mat\'ern-3/2 multi-output GP trained on the same data, except projected onto a latitude-longitude map.

Results are shown in Figure \ref{fig:aeolus}, where the benefits of using a manifold vector-valued GP become clear. 
When the satellite crosses the left/right boundary in the lat/lon projection, the outputs from the Euclidean vector-valued GP give rise to a spurious discontinuity in the uncertainty along the solid pink line. 
In addition, predictions become less certain in the Euclidean case as the satellite approaches the poles, which is simply an artifact of the distortion caused by projecting the spherical data onto the plane.
By construction, the manifold vector-valued GP is able to avoid both of these issues, resulting in a more realistic prediction with much more uniform uncertainty along the satellite trajectory from pole to pole. In addition, the predictions from the manifold GP are more certain overall, due to the useful structural bias embedded in the kernel.

\section{Conclusion}
In this paper, we propose techniques that generalize Gaussian processes to model vector fields on Riemannian manifolds.
This is done by first providing a well-defined notion of such processes on manifolds
and then introducing an explicit method to construct them in a way that respects the underlying geometry. 
By virtue of satisfying the key condition of gauge independence, our construction is coordinate-free and thus meaningful on manifolds.
In addition to this, we extend standard Gaussian process training methods, such as variational inference, to this setting, and verify that such methods are also compatible with gauge independence. 
This theoretical work gives practitioners additional tools for stochastic modeling of vector fields on manifolds. 
As such, its societal impact will be mainly determined by the applications that belong to the domain of future work.
We demonstrate our techniques on a series of examples in modeling dynamical systems and weather science, and show that incorporating geometric structural bias into probabilistic modeling is beneficial in these settings to obtain coherent predictions and uncertainties.

\section{Acknowledgements and Funding Disclosure}
Michael is supported by the EPSRC Centre for Doctoral Training in Modern Statistics and Statistical Machine Learning (EP/S023151/1).
Section~\ref{sec:model construction} and Appendix~\ref{apdx:theory} of this paper were solely financially supported by the RSF grant N\textsuperscript{\underline{o}}21-11-00047 (\url{https://rscf.ru/en/project/21-11-00047/}), but was contributed to by all authors in equal proportion with the remaining work.
Other sections were supported by the remaining funding. 
We also acknowledge support from Huawei Research.
We thank Peter Dueben for useful suggestions in Section \ref{sec:weather-modeling}, including, in particular, making us aware of the Aeolus satellite.
We are grateful to Nick Sharp for introducing us to the process used for making the three-dimensional figures appearing in this work.

\printbibliography

\newpage

\appendix

\section{Theory}
\label{apdx:theory}

\subsection*{Preliminaries on Gaussian measures}

Since we are working in a setting beyond $\R^d$, we need a suitable notion of a multivariate Gaussian that can be employed in a coordinate-free manner.
We employ the notion of a Gaussian in the sense of duality, given below.
These notions are standard and classical, but since they are not well-known in machine learning, and for completeness, we prove the necessary properties ourselves.

\begin{definition}
\label{def:gauss-dual}
Let $(\Omega,\c{F},\P)$ be a probability space.
Let $V$ by a finite-dimensional real topological vector space, equipped with the standard topology, Borel $\sigma$-algebra, and the canonical pairing $\dualprod{\.}{\.} : V^* \x V \-> \R$ with its topological dual $V^*$.
A random vector $v : \Omega \-> V$ is called \emph{Gaussian} if, for all $\phi \in V^*$, the random variable $\dualprod{\phi}{v} : \Omega \-> \R$ is univariate Gaussian.
\end{definition}

\paragraph{Remark.} It is not hard to show that in the setting of the definition above, the random variables~$\dualprod{\phi_1}{v}, \ldots, \dualprod{\phi_k}{v}$ are jointly Gaussian for any finite collection $\phi_1, \ldots, \phi_k \in V^*$.
Indeed, this is equivalent to the Gaussianity of every linear combination $\alpha_1 \dualprod{\phi_1}{v} + \ldots + \alpha_k \dualprod{\phi_k}{v} = \dualprod{\alpha_1 \phi_1 + \ldots + \alpha_k \phi_k}{v}$, which is also ensured by the definition since $\alpha_1 \phi_1 + \ldots + \alpha_k \phi_k \in V^*$.

We begin by showing that a Gaussian random vector in the sense of duality is characterized by a mean and a covariance, just like Gaussians in the standard, coordinate-dependent sense, starting with defining appropriate analogs of both notions in this setting.

\begin{lemma} \label{thm:generalized_moments1}
For every Gaussian random vector $v$, there is a unique vector $\mu \in V$ and unique symmetric positive semi-definite bilinear form $k : V^* \x V^* \-> \R$ such that for all $\phi \in V^*$, we have $\E \dualprod{\phi}{v} = \dualprod{\phi}{\mu}$ and $k(\phi, \psi) = \Cov(\dualprod{\phi}{v},\dualprod{\psi}{v})$.
We say that $\mu$ is its \emph{mean} and $k$ is its \emph{covariance form}, and write $v \~[N](\mu, k)$.
\end{lemma}

\begin{proof}
Consider the map $\E\dualprod{\.}{v} : V^* \-> \R$.
This map is a linear functional on the space $V^*$.
Since $V$ is finite-dimensional, $V$ is reflexive, so there is exactly one vector $\mu \in V$ such that 
\[
\dualprod{\phi}{\mu} = \E\dualprod{\phi}{v}
\]
for all $\phi\in V^*$. 
Next, define $k$ as
\[
k(\phi,\psi) = \Cov(\dualprod{\phi}{v},\dualprod{\psi}{v})
\]
for all $\phi,\psi\in V^*$. Clearly, $k$ is bilinear and positive semi-definite, that is $k(\phi, \phi) \geq 0$ for all $\phi \in V^*$. Thus the claim follows.
\end{proof}

This tells us that every Gaussian random vector admits a mean and covariance: we now show that such Gaussians exist and are uniquely determined by this pair.
Recall that for a measure $\pi$, and a measurable function $\phi$, the \emph{pushfoward measure} $\phi_* \pi$ is defined as $(\phi_*\pi)(A) = \pi(\phi^{-1}(A))$ for all measurable sets $A$.

\begin{lemma} \label{thm:generalized_moments2}
For any vector $\mu \in V$ and any positive semi-definite bilinear form  $k: V^* \x V^* \-> \R$, there exists a random vector $v \~[N](\mu, k)$.
Moreover, if $w: \Omega \-> V$ is another Gaussian random vector in the sense of Definition~\ref{def:gauss-dual} with $w \~[N](\mu, k)$, then $v$ and $w$ are identically distributed.
\end{lemma}

\begin{proof}
Choose a basis $(e_i)$ on $V$, and let $(e^i)$ be the dual basis.
Define the vector $\v\mu \in \R^d$ and matrix $\m{K} \in \R^{d\x d}$ by
\[
\v\mu &= \begin{bmatrix}
\dualprod{e^1}{\mu}
\\
\vdots
\\
\dualprod{e^d}{\mu}
\end{bmatrix}
&
\m{K} &= \begin{bmatrix}
k(e^1, e^1) & \hdots & k(e^1, e^d)
\\
\vdots & \ddots & \vdots
\\
k(e^d, e^1) & \hdots & k(e^d, e^d)
\end{bmatrix}  
.  
\]
By positive semi-definiteness of $k$, the matrix $\m{K}$ is a positive semi-definite matrix, so there exists a random vector $\v{v}\~[N](\v\mu,\m{K})$ in the classical Euclidean sense.
Let $\c{E} : V \-> \R^n$ be the continuous linear isomorphism induced by the basis and define
\[
v = \c{E}^{-1}\v{v}.
\]

We claim that (a) $v$ is Gaussian, that is, if we test it against any covector, we obtain a univariate Gaussian, (b) the mean vector of $v$ is $\mu$, and (c) the covariance form of $v$ is $k$. To show (a), let $v^i$ denote the components of $\v{v}$ (scalar Gaussian random variables) so that $v = \sum_{i=1}^d v^i e_i$ and for any $\phi \in V^*$, write $\phi = \sum_{i=1}^d \phi_i e^i$, where $\phi_i = \dualprod[1]{\phi}{e_i}$. Then we have
\[
\dualprod[1]{\phi}{v} = \dualprod{\sum_{i=1}^d \phi_i e^i}{\sum_{j=1}^d v^j e_j} = \sum_{i=1}^d \sum_{j=1}^d \phi_i v^j \ubr{\dualprod[1]{e^i}{e_j}}_{\delta_{ij}} = \sum_{i=1}^d \phi_i v^i.
\]
Since each $v^i$ is a univariate Gaussian, the linear combination on the right hand side is also a univariate Gaussian, which proves (a). To prove (b) and (c), we see that for any $\phi \in V^*$,
\[
\E \dualprod[1]{\phi}{v}
&= \E \dualprod[1]{\phi}{\sum_{i=1}^d v^i e_i}
=
\E \sum_{i=1}^d v^i \dualprod{\phi}{e_i}
\\
&=
\sum_{i=1}^d \ubr{\del{\E v^i}}_{\dualprod{e^i}{\mu}} \dualprod{\phi}{e_i}
=
\dualprod[1]{\phi}{\sum_{i=1}^d \dualprod[0]{e^i}{\mu} e_i}
=
\dualprod{\phi}{\mu}.
\]
Thus $v$ has the right mean. Now take an additional $\psi \in V^*$ and write
\[
\Cov\del[2]{\dualprod[1]{\phi}{v}, \dualprod[1]{\psi}{v}}
&=
\E
\del{
\del{
\dualprod[1]{\phi}{v}
-
\dualprod[1]{\phi}{\mu}
}
\del{
\dualprod[1]{\psi}{v}
-
\dualprod[1]{\psi}{\mu}
}
}
\\
&=
\E
\del[4]{
\sum_{i=1}^d \del{v^i - \dualprod{e^i}{\mu}} \dualprod{\phi}{e_i}
}
\del[4]{
\sum_{j=1}^d \del{v^j - \dualprod{e^j}{\mu}} \dualprod{\psi}{e_j}
}
\\
&=
\sum_{i=1}^d \sum_{j=1}^d
\dualprod{\phi}{e_i}
\ubr{\E \del{\del{v^i - \dualprod[0]{e^i}{\mu}} \del{v^j - \dualprod[0]{e^j}{\mu}}}}_{k(e^i, e^j)}
\dualprod{\psi}{e_j}
\\
&=
k\del[4]{\sum_{i=1}^d \dualprod{\phi}{e_i} e^i, \sum_{j=1}^d \dualprod{\psi}{e_j} e^j}
=
k(\phi, \psi),
\]
hence $v$ has the right covariance form.

Now let $w : \Omega \-> V$ be another Gaussian random vector with $w \~[N](\mu,k)$, and let $\pi_w$ be its pushforward measure.
Similarly, let $\pi_v$ be the pushforward measure of $v$.
Reversing the above argument, we see that pushforwards of measures $\pi_v$ and $\pi_w$ through $\c{E}$, which we denote by $\pi_{\v{v}}$ and $\pi_{\v{w}}$, are both Gaussian distributions (in the classical sense) in $\R^d$ with the same mean vectors $\v{\mu}$ and covariance matrices $\m{K}$.
Hence $\pi_{\v{v}} = \pi_{\v{w}}$ in distribution, but since $\c{E}$ is a measurable space isomorphism,\footnote{A measurable space isomorphism is a measurable bijection with a measurable inverse.} we have $\pi_v = \pi_w$, which proves the claim.
\end{proof}

Lemmas~\ref{thm:generalized_moments1}~and~\ref{thm:generalized_moments2} show that a pair $\mu, k$ defines a unique probability distribution on $V$ which we call the Gaussian distribution with mean vector $\mu$ and covariance form $k$ on the vector space $V$ and denote by $\operatorname{N}(\mu, k)$.
This establishes a notion of Gaussianity that is suitable and natural for describing finite-dimensional marginals in a coordinate-free manner.

\subsection*{Existence and uniqueness (Proof of Theorem \ref{prop:riem-kern})}

Here, we prove that Gaussian vector fields exist and are uniquely determined by their mean vector field and cross-covariance kernel.
Our goal now is, from a cross-covariance kernel, to construct a projective family of finite-dimensional marginals.

\begin{definition}[Preliminaries]
Let $X$ be a smooth manifold. 
Let
\[
\Gamma_{\f{nns}}(TX) = \{f : X \-> TX : \proj_X \after f = \id_X\}
\]
be the vector space of not necessarily smooth sections.
\end{definition}

\begin{definition}[Cross-covariance kernel]
A symmetric function $k : T^*X \x T^*X \-> \R$ is called 
\emph{fiberwise bilinear} if at any pair of points $x, x' \in X$, we have
\[
k(\lambda \alpha_x + \mu \beta_x, \gamma_{x'}) &= \lambda k(\alpha_x, \gamma_{x'}) + \mu k(\beta_x, \gamma_{x'})
\]
for any $\alpha_x, \beta_x \in T^*_x X$, $\gamma_{x'} \in T^*_{x'} X$ and $\lambda, \mu \in \R$, 
where we note by symmetry that the same requirement applies to its second argument.
A fiberwise bilinear function $k$ is called \emph{positive semi-definite} if for any set of covectors $\alpha_{x_1}, \ldots, \alpha_{x_n} \in T^*X$, we have
\[ \label{eqn:fiber_pd}
\sum_{i=1}^n\sum_{j=1}^n k(\alpha_{x_i}, \alpha_{x_j}) \geq 0
.
\]
We call a symmetric fiberwise bilinear positive semi-definite function a \emph{cross-covariance kernel}.
\end{definition}

We show in the following example that this definition of the cross-covariance kernel is compatible with the notion of matrix-valued kernels used in classical vector-valued GPs and extends it naturally.

\begin{example}[Euclidean case]
Consider $X = \R^d$ with a fixed inner product and an orthonormal basis, under which $\R^d$ is identified with $\del{\R^d}^*$.
Consider a matrix-valued kernel $\kappa: \R^d \x \R^d \-> \R^{d \x d}$ in the standard sense.
Let $k((x, v), (x', v')) = v^T \kappa(x, x') v'$. Then $k: T^* \R^d \x T^* \R^d \-> \R$ is a cross-covariance kernel in the above sense.

Indeed, $k$ is symmetric and fiberwise bilinear. Moreover, since $\kappa$ is positive semi-definite in the regular sense, we have that for arbitrary $x_1, \dots, x_n \in \R^d$, the $n d \x n d$ matrix
\[
\Gamma(x_1, \dots, x_n)
=
\begin{bmatrix}
\kappa(x_1, x_1) & \dots  & \kappa(x_1, x_n) \\
\vdots           & \ddots & \vdots \\
\kappa(x_n, x_1) & \dots  & \kappa(x_n, x_n)
\end{bmatrix}
\]
is positive semi-definite, meaning that for an arbitrary collection $v_1, \dots, v_n \in \R^d$, we have
\[
0 \leq
\begin{bmatrix}
v_1^T & \dots & v_n^T
\end{bmatrix}
\begin{bmatrix}
\kappa(x_1, x_1) & \dots  & \kappa(x_1, x_n) \\
\vdots           & \ddots & \vdots \\
\kappa(x_n, x_1) & \dots  & \kappa(x_n, x_n)
\end{bmatrix}
\begin{bmatrix}
v_1 \\
\vdots \\
v_n
\end{bmatrix}
=
\sum_{i=1}^n \sum_{j=1}^n \ubr{v_i^T \kappa(x_i, x_j) v_j}_{k((x_i, v_i), (x_j, v_j))}.
\]
Condition \eqref{eqn:fiber_pd} thus follows, proving that this is a valid cross-covariance kernel.
\end{example}

We proceed to introduce the system of coordinate-free finite-dimensional marginals that will be used to construct the vector-valued GP.

\begin{definition}
Let $\mu \in \Gamma_{\f{nns}}(TX)$ and $k : T^*X \x T^*X \-> \R$ be a cross-covariance kernel.
For any $x_1,\ldots,x_n \in X$, let $V_{x_1,\ldots,x_n} = T_{x_1} X \oplus \ldots \oplus T_{x_n} X$ and $V^*_{x_1,\ldots,x_n} = T^*_{x_1} X \oplus \ldots \oplus T^*_{x_n} X$.
Define $\mu_{x_1,\ldots,x_n} \in V_{x_1,\ldots,x_n}$ and $k_{x_1,\ldots,x_n} : V_{x_1,\ldots,x_n}^* \x V_{x_1,\ldots,x_n}^* \-> \R$
by
\[
\mu_{x_1,\ldots,x_n} &= (\mu(x_1),\ldots,\mu(x_n))
&
k_{x_1,\ldots,x_n}(\alpha,\beta) &= \sum_{i=1}^n\sum_{j=1}^n k(\alpha_{x_i}, \beta_{x_j})
\]
for any $\alpha = (\alpha_{x_1}, \ldots, \alpha_{x_n})$, $\beta = (\beta_{x_1}, \ldots, \beta_{x_n}) \in V^*_{x_1, \ldots, x_n}$.
We denote $\pi_{x_1,\ldots,x_n} = \operatorname{N}(\mu_{x_1,\ldots,x_n}, k_{x_1,\ldots,x_n})$ the system of marginals induced by $k$.
\end{definition}

We now prove existence and uniqueness of a measure on $\Gamma_{\f{nns}}(TX)$ from the Gaussian measures defined on $V_{x_1,\ldots,x_n}$ for any $\{x_1, \ldots, x_n\} \subseteq X$.
We do this by means of the general form of the Kolmogorov extension theorem formulated below.
Recall again that for a measure $\pi$, and a measurable function $\phi$, the \emph{pushfoward measure} $\phi_* \pi$ is defined as $(\phi_*\pi)(A) = \pi(\phi^{-1}(A))$ for all measurable sets $A$.

\begin{result}[Kolmogorov Extension Theorem] \label{result:tao}
Let $(X_\alpha,\c{B}_\alpha,\c{O}_\alpha)_{\alpha\in A}$ be a family of measurable spaces, each equipped with a topology.
For each finite $B \subseteq A$, let $\mu_B$ be an inner regular probability measure on $X_B = \prod_{\alpha\in B} X_\alpha$ with $\sigma$-algebra $\c{B}_B$ and with the product topology $\c{O}_B$ obeying
\[\label{project-measure}
\del{\proj_C}_* \mu_B = \mu_C
\]
whenever $C \subseteq B \subseteq A$ are two nested finite subsets of $A$. Here projections $\proj_C: X_B \-> X_C$ are defined by $\proj_C(\cbr{x_\alpha}_{\alpha \in B}) = \cbr{x_\alpha}_{\alpha \in C}$ and $\del{\proj_C}_*$ denotes the pushforward by $\proj_{C}$.
Then there exists a unique probability measure $\mu_A$ on $\c{B}_A$ with the property that $\del{\proj_B}_* \mu_A = \mu_B$ for all finite $B \subseteq A$.
\end{result}

\begin{proof}
\textcite[Theorem 2.4.3]{tao11}.
\end{proof}

By showing the existence of a probability measure on the space $\Gamma_{\f{nns}}(TX)$, one can start speaking about random variables $f : \Omega \-> \Gamma_{\f{nns}}(TX)$ with said measure as their distribution: these are the Gaussian vector fields we seek.
However, in order to apply the above result, we first need to verify condition \eqref{project-measure}. This is done in the following.

\begin{proposition} \label{prop:projective}
The family of measures $(\pi_{x_1,\ldots,x_n})_{\{x_1,\ldots,x_n\}\subseteq X}$ is a projective family in the sense that for any $\{x_1,\ldots,x_m\} \subseteq \{x_1,\ldots,x_n\} \subseteq X$, we have
\[
(\proj_{x_1,\ldots,x_m})_* \pi_{x_1,\ldots,x_n} = \pi_{x_1,\ldots,x_m}
\]
where $\proj_{x_1,\ldots,x_m} : V_{x_1,\ldots,x_n} \-> V_{x_1,\ldots,x_m}$ is the canonical projection induced by the direct sum.
\end{proposition}

\begin{proof}
Take two random variables $v_{x_1,\ldots,x_n} : \Omega \-> V_{x_1, \ldots, x_n}$ and $v_{x_1,\ldots,x_m} : \Omega \-> V_{x_1, \ldots, x_m}$ with $v_{x_1,\ldots,x_n} \sim \pi_{x_1,\ldots,x_n}$ and $v_{x_1,\ldots,x_m} \sim \pi_{x_1,\ldots,x_m}$.
It suffices to show that for the random variable $v_{x_1,\ldots,x_n} : \Omega \-> V_{x_1, \ldots, x_m}$ we have
\[
v_{x_1,\ldots,x_m} \stackrel{\d}{=} \proj_{x_1,\ldots,x_m} v_{x_1,\ldots,x_n}
\]
where $\stackrel{\d}{=}$ denotes the equality of distributions.
We first show that $\proj_{x_1,\ldots,x_m} v_{x_1,\ldots,x_n}$ is Gaussian.
Let $\phi\in V^*_{x_1,\ldots,x_m}$ and write 
\[
\dualprod{\phi}{\proj_{x_1,\ldots,x_m} v_{x_1,\ldots,x_n}} = \dualprod{(\phi, 0)}{v_{x_1,\ldots,x_n}}
\]
where $(\phi, 0) \in V^*_{x_1,\ldots,x_n}$ is the natural inclusion of $\phi \in V^*_{x_1,\ldots,x_m}$ in the space $V^*_{x_1,\ldots,x_n}$ by padding with the zero vector over all components of the direct sum whose indices are not $x_1,\ldots,x_m$.
This identity holds for all vectors, hence it holds for random vectors, and $\proj_{x_1,\ldots,x_m} v_{x_1,\ldots,x_n}$ is Gaussian.
Now, we compute its moments: write
\[
\E\dualprod{\phi}{\proj_{x_1,\ldots,x_m} v_{x_1,\ldots,x_n}} &= \E\dualprod{(\phi, 0)}{v_{x_1,\ldots,x_n}} &
\\
&= \dualprod{(\phi, 0)}{\mu_{x_1,\ldots,x_n}} 
\\
&=  \dualprod{\phi}{\proj_{x_1,\ldots,x_m} \mu_{x_1,\ldots,x_n}} 
\\
&= \dualprod{\phi}{\mu_{x_1,\ldots,x_m}}
\]
where the last line follows by definition of $\mu_{x_1,\ldots,x_m}$, and
\[
\Cov (\dualprod{\phi}{\proj_{x_1,\ldots,x_m} v_{x_1,\ldots,x_n}},\,&\dualprod{\psi}{\proj_{x_1,\ldots,x_m} v_{x_1,\ldots,x_n}}) 
\\
&= \Cov \left(\dualprod{(\phi, 0)}{v_{x_1,\ldots,x_n}},\dualprod{(\psi, 0)}{v_{x_1,\ldots,x_n}}\right)
\\
&= k_{x_1,\ldots,x_n}((\phi, 0), (\psi, 0))
\\
&= k_{x_1,\ldots,x_m}(\phi, \psi)
\]
where the last line follows by bilinearity and the definition of $k_{x_1,\ldots,x_m}$.

So far we have shown that $\proj_{x_1,\ldots,x_m} v_{x_1,\ldots,x_n}$ is Gaussian over $V_{x_1,\ldots,x_m}$ and its mean vector and covariance form coincide with those of $v_{x_1,\ldots,x_m}$.
Hence, by the uniqueness part of Lemma~\ref{thm:generalized_moments2} we have $v_{x_1,\ldots,x_m} \stackrel{d}{=} \proj_{x_1,\ldots,x_m} v_{x_1,\ldots,x_n}$.
This finishes the proof.
\end{proof}

We are now ready to apply the Kolmogorov extension theorem to show existence of the desired distribution.

\begin{proposition}
There exists a unique measure $\pi_\infty$ on the infinite product space $\prod_{x\in X} T_x X$.\footnote{Note that this is the Tychonoff product of topological spaces rather than a direct product of linear spaces.}
\end{proposition}

\begin{proof}
We apply the prior result \ref{result:tao}.
Let $X$ be the index set, and take $(T_x X)_{x\in X}$, equipped with the standard topology and Borel $\sigma$-algebra as our measurable spaces.
For each finite $\{x_1,\ldots,x_n\}\subseteq X$, take $\pi_{x_1,\ldots,x_n}$ as our probability measure, and note that since each $\pi_{x_1,\ldots,x_n}$ is a finite measure on a finite-dimensional real vector space $V_{x_1,\ldots,x_n}$, it is automatically inner regular.
Moreover, the family of measures $(\pi_{x_1,\ldots,x_n})_{\{x_1,\ldots,x_n\}\subseteq X}$ is projective by Proposition \ref{prop:projective}. The claim follows.
\end{proof}

This gives our GP as a measure on an infinite Cartesian space: we now map this measure into the space of sections.

\begin{corollary} \label{thm:exists_unique}
There exists a unique measure $\pi_{\Gamma_{\f{nns}}(TX)}$ on $\Gamma_{\f{nns}}(TX)$ equipped with the pushforward $\sigma$-algebra.
\end{corollary}

\begin{proof}
Define the operator $\c{I} : \prod_{x\in X} T_x X \-> \Gamma_{\f{nns}}(TX)$ by 
\[
(\c{I}s)(x) = (x,\proj_xs)
\]
for all $x \in X$ and $s \in \prod_{x\in X} T_x X$.
Take $\pi_{\Gamma_{\f{nns}}(TX)} = \c{I}_* \pi_\infty$.
\end{proof}

This is the probability distribution of our Gaussian process. 
We are now ready to define Gaussian vector fields, and show that each Gaussian vector field in turn possesses a mean vector field and cross-covariance kernel.

\begin{definition}
\label{def:vect-gp2}
Let $X$ be a manifold.
We say that a random vector field $f : \Omega \-> \Gamma_{\f{nns}}(TX)$ is \emph{Gaussian} if for any finite set of locations $(x_1,\ldots,x_n) \in X^n$, the random vector $f(x_1),\ldots,f(x_n) \in T_{x_1}X \oplus \ldots \oplus T_{x_n}X$ is Gaussian in the sense of Definition \ref{def:gauss-dual}.
\end{definition}

\begin{definition} \label{def:mean-and-kernel}
Let $f : \Omega \-> \Gamma_{\f{nns}}(TX)$ be a Gaussian vector field.
Define $\mu$ to be the unique vector field for which, for any $x \in X$ and any $\phi \in T^*_{x}$, we have that
\[
\dualprod{\phi}{\mu(x)} = \E \dualprod{\phi}{f(x)}
.
\]
Next taking an additional $x' \in X$ and $\psi \in T^*_{x'}$, define the cross-covariance kernel $k$ by
\[
k(\phi,\psi) = \Cov(\dualprod{\phi}{f(x)},\dualprod{\psi}{f(x')})
.
\]
\end{definition}

Summarizing, we obtain the following claim.

\begin{theorem}
Every pair consisting of a mean vector field and symmetric fiberwise bilinear positive definite function $k : T^*X \x T^*X \-> \R$, which we call a \emph{cross-covariance kernel}, defines a unique (distribution-wise) Gaussian vector field in the sense of Definition \ref{def:vect-gp2}.
Conversely, every Gaussian vector field admits and is characterized uniquely by this pair.
\end{theorem}

\begin{proof}
Corollary~\ref{thm:exists_unique}, Definition~\ref{def:vect-gp2}, and Definition~\ref{def:mean-and-kernel}.
\end{proof}

\subsection*{Embeddings (Proof of Proposition \ref{prop:embedding})}

\begin{proposition}
Let $\f{emb} : X \-> \R^p$ be an embedding, let $f$ be a Gaussian vector field on $X$, and denote by $\v{f}_{\f{emb}} : \f{emb}(X) \-> \R^p$ its pushforward along the embedding, that is, for any $x \in X$,
\[
\v{f}_{\f{emb}}(\f{emb}(x)) = \d_x\!\f{emb}(f(x)),
\]
where $\d_x\!\f{emb} : T_x X \-> T_{\f{emb}(x)} \R^p$ is the differential of $\f{emb}$.
Then $\v{f}_{\f{emb}}$ is a vector-valued Gaussian process in the standard sense.
\end{proposition}

\begin{proof}
Let $x_1,\ldots,x_n \in X^n$ be a finite set of arbitrary locations. In what follows, we use a slight abuse of notation by letting $x_i$ denote both $x_i$ and $\f{emb}(x_i)$ for simplicity. We claim that the random vector $(\v{f}_{\f{emb}}(x_1), \ldots, \v{f}_{\f{emb}}(x_n)) \in \R^{np}$ is multivariate Gaussian, which is sufficient to prove our result. Since $f$ is a Gaussian vector field, we have that 
\[
(f(x_1),\ldots,f(x_n)) \~[N](\mu_{x_1,\ldots,x_n},k_{x_1,\ldots,x_n})
\]
is a Gaussian random vector on $T_{x_1} X \oplus \ldots \oplus T_{x_n} X$.
Now consider the map $\phi_{x_1,\ldots,x_n} : T_{x_1} X \oplus \ldots \oplus T_{x_n} X \-> T_{\f{emb}(x_1)} \R^p \oplus \ldots \oplus T_{\f{emb}(x_n)} \R^p \cong \R^{np}$ defined as
\[
\phi_{x_1,\ldots,x_n}(f_{x_1}, \ldots, f_{x_n}) = (\d_{x_1}\!\!\f{emb}(f_{x_1}), \ldots, \d_{x_n}\!\!\f{emb}(f_{x_n})),
\]
for all $(f_{x_1}, \ldots, f_{x_n}) \in T_{x_1} X \oplus \ldots \oplus T_{x_n} X$, which is linear, owing to the linearity of $\d_x\!\f{emb}$.
Since linear maps preserve Gaussianity, it follows that the vector $\phi_{x_1, \ldots, x_n} (f(x_1),\ldots,f(x_n)) = (\v{f}_{\f{emb}}(x_1), \ldots, \v{f}_{\f{emb}}(x_n)) \in \R^{np}$ is multivariate Gaussian and the claim follows.
\end{proof}

\subsection*{Coordinate Expressions (Proof of Proposition \ref{prop:matrix-representation})}

We recall the definition of a \emph{frame} on $X$ and its dual object, namely, the \emph{coframe}.
\begin{definition}
A \emph{frame} $F$ on $X$ is defined as a collection $(e_i)_{i=1}^d$ of not necessarily smooth sections of $TX$ such that at each point $x \in X$, the vectors $(e_i(x))_{i=1}^d$ form a basis of $T_xX$. The corresponding \emph{coframe} $F^*$ is defined as a collection $(e^i)_{i=1}^d$ of not necessarily smooth  sections of $T^*X$ such that $\left<e^i(x) | e_j(x)\right> = \delta_{ij}$ for all $x \in X$.
\end{definition}

\begin{proposition}\label{prop:matrix-representation-appendix}
Let $f : \Omega \-> \Gamma_{\f{nns}}(TX)$ be a Gaussian vector field on $X$ with cross-covariance kernel ${k : T^*X \x T^*X \-> \R}$. Given a frame $F = (e_1, \ldots, e_d)$ on $X$ and $F^* = (e^1, \ldots, e^d)$ be its coframe, define $f^i = \dualprod[1]{e^i}{f}$ for all $i = 1, \ldots, d$.  Then $\v{f} = (f^1, \ldots, f^d) : \Omega \x X \-> \R^d$ is a vector-valued GP in the usual sense with matrix-valued kernel ${\m{K}_F : X \x X \-> \R^{d \x d}}$ given by
\[
\m{K}_F(x, x') =
\begin{bmatrix}
k(e^1(x), e^1(x')) & \dots  & k(e^1(x), e^d(x')) \\
\vdots           & \ddots & \vdots \\
k(e^d(x), e^1(x')) & \dots  & k(e^d(x), e^d(x'))
\end{bmatrix}.
\]
Conversely, given a vector-valued GP $\v{f} = (f^1, \ldots, f^d): \Omega \x X \-> \R^d$ and a frame $F = (e_1, \ldots, e_d)$ on $X$, $f(\cdot) := \sum_{i=1}^d f^i(\cdot) e_i(\cdot)$ defines a Gaussian vector field on $X$.
\end{proposition}

\begin{proof}
First, we note that $f^i(x) = \dualprod[1]{e^i(x)}{f(x)}$ are jointly Gaussian for all $i=1, \ldots, d$ and all $x \in X$. Thus for any $x_1, \ldots, x_n \in X$, the vector $(\v{f}(x_1), \ldots, \v{f}(x_n)) \in \R^{n \x d}$ is multivariate Gaussian and therefore $\v{f}$ is a vector-valued GP in the usual sense. Now for any $x, x' \in X$, the kernel of $\v{f}$ evaluated at these points reads
\[
\m{K}_F(x, x') &= \begin{bmatrix}
\Cov(f^1(x), f^1(x')) & \dots  & \Cov(f^1(x), f^d(x')) \\
\vdots           & \ddots & \vdots \\
\Cov(f^d(x), f^1(x')) & \dots  & \Cov(f^d(x), f^d(x'))
\end{bmatrix} \\
&=
\begin{bmatrix}
k(e^1(x), e^1(x')) & \dots  & k(e^1(x), e^d(x')) \\
\vdots           & \ddots & \vdots \\
k(e^d(x), e^1(x')) & \dots  & k(e^d(x), e^d(x'))
\end{bmatrix},
\]
which follows from Definition \ref{def:mean-and-kernel}. This concludes the first part of the proof.

To prove the converse direction, for any collection of points $x_1, \ldots, x_n \in X$, define the random vector $v_{x_1, \ldots, x_n} = (f(x_1), \ldots, f(x_n))$, where $f$ is given by $f(x) = \sum_{i=1}^d f^i(x) e_i(x)$. Now for any $\phi_{x_1, \ldots, x_n} = (\phi_{x_1}, \ldots, \phi_{x_n}) \in V_{x_1, \ldots, x_n}^*$, we have
\[
\dualprod[1]{\phi_{x_1, \ldots, x_n}}{v_{x_1, \ldots, x_n}} &= \sum_{i=1}^n \dualprod[1]{\phi_{x_i}}{f(x_i)} \\
&= \sum_{i=1}^n \dualprod[1]{\phi_{x_i}}{\sum_{j=1}^d f^j(x_i) e_j(x_i)} \\
&= \sum_{i=1}^n \sum_{j=1}^d f^j(x_i) \dualprod[1]{\phi_{x_i}}{e_j(x_i)}.
\]
Since $f^j(x_i)$ is univariate Gaussian for all $i = 1, \ldots, n$ and $j = 1, \ldots, d$, the above linear combination is univariate Gaussian and therefore $v_{x_1, \ldots, x_n}$ is Gaussian in the sense of Definition \ref{def:gauss-dual}. Since $x_1, \ldots, x_n$ were chosen arbitrarily, $f$ is a Gaussian vector field.
\end{proof}

\subsection*{Gauge Independence (Proof of Corollary \ref{cor:equivariance})}

Given two frames $F, F'$ on $X$, an abstract vector $f_x \in T_xX$ has two vector representations $\v{f}_x, \v{f}'_x$ in the respective frames. Recall that $F'$ is said to be obtained from $F$ by a \emph{gauge transformation} with respect to a matrix field $\m{A}: X \-> \f{GL}(d, \R)$, if
\[
\v{f}'_x = \m{A}(x) \v{f}_x
\]
holds for all $x \in X$, and we write $F' = \m{A}F$. In the following, we compute an explicit expression for the gauge-transformed frame $\m{A}F$ and its coframe.

\begin{lemma}
Let $F = (e_1, \ldots, e_d)$ be a frame on $X$, $\m{A}: X \-> \f{GL}(d, \R)$ be a matrix field of gauge transformations, $\m{A}F = (\eps_1, \ldots, \eps_d)$ be the gauge transformed frame as above and let $(\m{A}F)^* = (\eps^{1}, \ldots, \eps^{d})$ be the corresponding coframe. Then we have the following explicit expressions
\[
\eps_i(x) &= \sum_{j=1}^d e_j(x) [\m{A}^{-1}(x)]_{ji}, & \eps^i(x) &= \sum_{j=1}^d [\m{A}(x)]_{ij} e^j(x).
\]
\end{lemma}
\begin{proof}
For any $x \in X$, let $f_x \in T_xX$ be an abstract vector, which has the vector representations $\v{f}_x$ and $\m{A}(x)\v{f}_x$ in the frames $F$ and $\m{A}F$ respectively. Letting $\v{f}_x = (f_x^1, \ldots, f_x^d)$, we have
\[
f_x = \sum_{i=1}^d f_x^i e_i(x) = \sum_{i=1}^d \sum_{j=1}^d ([\m{A}(x)]_{ji}f_x^i) \eps_j(x) &= \sum_{i=1}^d f_x^i \del{\sum_{j=1}^d \eps_j(x) [\m{A}(x)]_{ji}}.
\]
Thus, $e_i(x) = \sum_{j=1}^d \eps_j(x) [\m{A}(x)]_{ji}$, or identically, $\eps_i(x) = \sum_{j=1}^d e_j(x) [\m{A}^{-1}(x)]_{ji}$.
We now claim that $\eps^i(x) = \sum_{j=1}^d [\m{A}(x)]_{ij} e^j(x)$, which we prove by showing that it satisfies the relation $\dualprod{\eps^i(x)}{\eps_j(x)} = \delta_{ij}$ as follows:
\[
\dualprod{\eps^i(x)}{\eps_j(x)} &= \dualprod{\sum_{k=1}^d [\m{A}(x)]_{ik} e^k(x)}{\sum_{l=1}^d e_l(x) [\m{A}^{-1}(x)]_{lj}} \\
&= \sum_{k=1}^d \sum_{l=1}^d [\m{A}(x)]_{ik} \ubr{\dualprod{e^k(x)}{e_l(x)}}_{\delta_{kl}} [\m{A}^{-1}(x)]_{lj} \\
&= \sum_{k=1}^d [\m{A}(x)]_{ik} [\m{A}^{-1}(x)]_{kj} \\
&= \ubr{[\m{A}(x) \m{A}^{-1}(x)]_{ij}}_{\delta_{ij}}.
\]
This concludes the proof.
\end{proof}

The following is, then, straightforward to show.
\begin{corollary} \label{cor:matrix-equivariance-appendix}
Let $F$ be a frame on $X$ and $\m{K}_F : X \x X \-> \R^{d\x d}$ be the corresponding matrix representation of a cross-covariance kernel $k: T^*X \x T^*X \-> \R$. 
This satisfies the \emph{equivariance condition}
\[\label{eq:matrix-equivariance}
\m{K}_{\m{A} F}(x, x') = \m{A}(x) \m{K}_{F}(x,x') \m{A}(x')^T,
\]
where $\m{A} : X \-> \f{GL}(d, \R)$ is a gauge transformation applied to each point on $X$.
All cross-covariance kernels in the sense of Proposition~\ref{prop:riem-kern} arise this way.
\end{corollary}
\begin{proof}
Let $F = (e_1, \ldots, e_d)$ and $\m{A}F = (\eps_1, \ldots, \eps_d)$. Then by the previous lemma, we have
\[
[\m{K}_{\m{A}F}(x,x')]_{ij} &= k(\eps^i(x), \eps^j(x')) \\
&= \sum_{k=1}^d \sum_{l=1}^d k\left([\m{A}(x)]_{ik} \,e^k(x), [\m{A}(x')]_{jl} \,e^l(x')\right) \\
&= \sum_{k=1}^d \sum_{l=1}^d [\m{A}(x)]_{ik} \,[\m{K}_{F}(x,x')]_{kl} \,[\m{A}(x')]_{jl},
\]
which proves the identity \eqref{eq:matrix-equivariance}.

The second claim is obvious: take some cross-covariance kernel in the sense of Proposition~\ref{prop:riem-kern} and some frame---this induces a gauge independent matrix-valued kernel that correspond to the cross-covariance kernel in the sense of Proposition~\ref{prop:riem-kern} from which it was constructed in the first place.
\end{proof}

\subsection*{Projected Kernels (Proof of Proposition \ref{prop:projected-kernel})}

Here we formally describe the projected kernel construction.
We start by noting some properties of the projection matrices associated with differentials of isometric embeddings.

\begin{lemma}\label{lemma:PTP=Gamma}
Let $(X, g)$ be a Riemannian manifold and $\f{emb} : X \-> \R^{d'}$ be an isometric embedding. Given a frame $F = (e_1, \ldots, e_d)$ on $X$, denote by $\m{P}_{(\.)} : X \-> \R^{d \x d'}$ its associated projection matrix, defined for every $x$ as the matrix representation of $\d_x\!\f{emb}$ within $F$,  and $\m{\Gamma} : X \-> \R^{d \x d}$, the matrix field representation of the Riemannian metric $g$, that is, $[\m{\Gamma}(x)]_{ij} = g_x(e_i(x), e_j(x))$ for all $i, j = 1, \ldots, d$ and $x \in X$. Then we have
\[
\m{P}_x\m{P}_x^T = \m{\Gamma}(x).
\]
\end{lemma}
\begin{proof}
Since the embedding is isometric, for any $v, v' \in T_x X$, we have
\[
g_x(u,v) = \innerprod{\d_x\!\f{emb}(v)}{\d_x\!\f{emb}(v')},
\]
which, in the corresponding vector representation with respect to a frame $F$, reads
\[
\v{v}^T \m{\Gamma}(x) \v{v}' = \innerprod{\m{P}_x^T \v{v}}{\m{P}_x^T \v{v}'} = \v{v}^T (\m{P}_x\m{P}_x^T) \v{v}'.
\]
for any $\v{v},\v{v}'$. This implies that $\m{\Gamma}(x) = \m{P}_x\m{P}_x^T$ for all $x$ and proves the claim.
\end{proof}

We proceed to describe the projected kernel construction, which lets us transform a matrix-valued kernel on an ambient space into a cross-covariance kernel on the manifold.

\begin{proposition} \label{prop:projected-kernel-appendix}
Let $(X, g)$ be a Riemannian manifold, $\f{emb} : X \-> \R^{d'}$ be an isometric embedding and $F$ be a frame on $X$.
We denote by $\m{P}_{(\.)} : X \-> \R^{d \x d'}$ the associated projection matrix under $F$, and let $\v{f}' : X \-> \R^{d'}$ be any vector-valued Gaussian process with matrix-valued kernel $\v\kappa : X \x X \-> \R^{d' \x d'}$.
Then, the vector-valued function $\v{f} = \m{P}\v{f}' : X \-> \R^d$ defines a Gaussian vector field $f$ on $X$ using the construction in Proposition \ref{prop:matrix-representation-appendix}, whose kernel under the frame $F$ has matrix representation
\[\label{eq:projection-relation-appendix}
\m{K}_F(x, x') = \m{P}_x \v\kappa(x,x') \m{P}_{x'}^T.
\]
Moreover, all cross-covariance kernels $k : T^*X \x T^*X \-> \R$ arise this way.
\end{proposition}

\begin{proof}
We demonstrate the first part by computing the covariance of $\v{f}$. For any $x, x' \in X$, we have
\[
\m{K}_F(x, x')_{ij} &= \Cov(f^i(x), f^j(x'))\\
&= \Cov([\m{P}_x\v{f}'(x)]_i, [\m{P}_{x'} \v{f}'(x')]_j) \\
&= \sum_{k=1}^d \sum_{l=1}^d \Cov([\m{P}_x]_{ik} \,f'_k(x),  [\m{P}_{x'}]_{jl} \,f'_l(x')) \\
&= \sum_{k=1}^d \sum_{l=1}^d  \,[\m{P}_x]_{ik} \Cov(f_k'(x), f_l'(x')) \,[\m{P}_{x'}]_{jl} \\
&= \sum_{k=1}^d \sum_{l=1}^d \,[\m{P}_x]_{ik} \,\v\kappa(x, x')_{kl} \,[\m{P}_{x'}]_{jl},
\]
which proves the identity \eqref{eq:projection-relation-appendix}.

Conversely, let $k : T^*X \x T^*X \-> \R$ be a cross-covariance kernel. We first construct a matrix-valued kernel $\m{K}_F$ as in Proposition \ref{prop:matrix-representation-appendix}. Define
\[
\v\kappa(x, x') = \m{P}_x^T \m{K}_{\m{\Gamma}^{-1}F}(x,x') \m{P}_{x'},
\]
where $\m{\Gamma} : X \-> \R^{d \x d}$ is the matrix field representation of the metric $g$ as given in the statement of Lemma \ref{lemma:PTP=Gamma}.
Then by the same lemma, we have
\[
\m{P}_x \v\kappa(x,x') \m{P}_{x'}^T &= (\m{P}_x \m{P}_x^T) \m{K}_{\m{\Gamma}^{-1}F}(x,x') (\m{P}_{x'} \m{P}_{x'}^T) \\
&= \m{\Gamma}(x) \m{K}_{\m{\Gamma}^{-1}F}(x,x') \m{\Gamma}(x') \\
&= \m{K}_{F}(x,x'),
\]
where we used that $\m{\Gamma}(x)^T = \m{\Gamma}(x)$ and Corollary \ref{cor:matrix-equivariance-appendix} to deduce the last equality. Thus, any cross-covariance kernel $k$ can be obtained from a matrix-valued kernel $\v\kappa$ on the ambient space and therefore we do not lose any generality by working with the latter.
\end{proof}

\section{Experimental details}
\label{apdx:experiments}

Here, we include further details about the experiments conducted in Section \ref{sec:examples}.
All experiments were conducted on a single workstation with 64GB RAM, using CPU-based computation.

\subsection*{Fourier features for product kernels}

Throughout this paper we use the sparse GP formulation of \textcite{wilson2020pathwise,wilson2021pathwise} to work with GPs. In order to apply this method we need to be able to sample a Fourier feature approximation of the kernel. For stationary kernels supported on Euclidean space one typically uses a random Fourier feature (RFF) approximation \cite{rahimi_random_2008}
\[ 
\tl{f}(\cdot) &= \frac{1}{\sqrt{l}}\sum_{i=1}^l w_i \phi_i(x),
&
w_i &\~[N](0,1),
\]
where the $\phi_i$ are Fourier basis functions sampled from the spectral density of the kernel---see \textcite{sutherland15} for details.
The resulting random function $\tl{f}(\cdot)$ is then a Gaussian process with zero mean and kernel $l^{-1} \m{\Phi}(\cdot)^T \m{\Phi}(\cdot)$, where $\m{\Phi}$ is a vector of the $l$ basis functions. This approximates the true GP with a dimension-free error of the order $l^{-1/2}$. 

For kernels supported on compact spaces we use a Karhunen–Lo\'eve (KL) expansion. If we have a Gaussian process $f(\cdot)$ on a compact space, then we can optimally approximate this function (in terms of $L^2$-norm) by truncating its KL expansion
\[ 
f(\cdot) &= \sum_{i=1}^\infty w_i \psi_i(x) 
&
w_i &\~[N](0,\lambda_i)
\]
where $\phi_i, \lambda_i$ are the $i$\textsuperscript{th} eigenfunctions and values of the kernel, $\int_X \psi(x) k(x, \cdot) \d x = \lambda_i \psi_i (\cdot)$, sorted in descending order of the eingenvalues. For the squared exponential and Mat\'ern kernels on compact manifolds, these eigenfunctions are the eigenfunctions of the Laplacian of the manifold, and the eigenvalues are given by a transformation of the Laplacian eigenvalues \cite{Borovitskiy2020}.

The question then arises of what to do in the case of a product of kernels, each taking as input some different space, where some are suited to RFF approximation, and some to a KL approximation. We propose the following approach.
\1 All the RFF-appropriate kernels can be combined into one approximation by sampling the basis functions from the product measure of their Fourier transforms. 
\2 All of the KL-appropriate kernels can be combined into one approximation by computing the $k$ largest eigenvalues of the product manifold the kernels are defined on. If we have two compact manifolds with eigenvalue-function pairs $(\alpha_i, f_i(\cdot))_{i=1}^\infty$ and $(\beta_j, g_j(\cdot))_{j=1}^\infty$, then the eigenvalue-function pairs on the product manifold are $(\alpha_i + \beta_j, f_i(\cdot) g_j(\cdot))_{i,j=1}^{\infty,\infty}$ \cite{canzani_analysis_nodate}. We can repeatedly apply this to find the approximation for the kernel on arbitrary products of compact manifolds.
\3 Define the Fourier feature approximation of the combination of this RFF and KL approximations as 
\[ f(e, m) = \frac{1}{\sqrt{l}}\sum_{i=1}^l \sum_{j=1}^k w_{i,j} \phi_i(e) \psi_j(m) \quad \quad w_{i,j} \sim \mathcal{N}(0, \lambda_j) \]
where $e, m$ are the inputs to the RFF and KL appropriate kernels respectively, $\phi_i$ are the basis functions of the Euclidean kernels sampled from the product measure, and $\lambda_j, \psi_j,$ are the product eigenpairs on the product manifold. In the limit of infinite basis functions in both $l$ and $k$ this will give the correct kernel, and therefore the true prior.
\0

\subsection*{Dynamics experiment}

In this experiment, the base manifold is the state space of the single pendulum system. The position of the pendulum is represented by a single angle in $[0, 2\pi)$, which corresponds to the circle $\bb{S}^1$. The momentum lies then in its respective cotangent space. The phase space is the product of these, $\bb{S}^1 \x \R^1$.

This product manifold naturally embeds into $\R^3$ by embedding the circle into $\R^2$ in the canonical way, and leaving $\R^1$ unchanged. The embedding is then
\[
\f{emb}(q,p) = (\cos{q}, \sin{q}, p),
\]
where $q$ is the position and $p$ is the angular momentum. The global projection matrix given by
\[
\m{P}_{q, p} =
\begin{bmatrix}
-\sin{q} & \cos{q} & 0 \\
0 & 0 & 1
\end{bmatrix}.
\]
The Euclidean vector kernel we use is a separable kernel, produced by taking the product of an intrinsic squared exponential manifold kernel with the identity matrix to give a matrix valued kernel, \( \v\kappa = k_{\bb{S}^1 \x \R^1} \m{I}_{3 \x 3} \). The intrinsic manifold kernel is produced by the product of a typical Euclidean squared exponential kernel with a squared exponential kernel defined on $\bb{S}^1$ by \textcite{Borovitskiy2020}, so that \(k_{\bb{S}^1 \x \R^1} = k_{\bb{S}^1} k_{\R^1}\). The length scales of these kernels are set to 0.3 and 1.2 respectively, and the amplitude set to give $k(x,x) = 1$.

To learn the dynamics, we initialise the system at two start points, and evolve the system using leapfrog integration. From these observations of position, we backward Euler integrate the momentum of the system, $p_i = \frac{1}{2} ml( q_{i+1} - q_i)$, and from these position-momentum trajectories we estimate observations of the dynamics field
\[
\nabla_t ({q, p})_i = \del{\frac{q_{i+1} - q_i}{h} , \frac{p_{i+1} - p_i}{h}}
\]
where $h = 0.01$ is the step size.
Using these observations, we condition a sparse GP using all the data using the analytic expression for the sparse posterior kernel matrix. The result is an estimate of the system dynamics with suitable uncertainty estimates. In order to compute rollouts of these dynamics, we follow \textcite{wilson2020pathwise,wilson2021pathwise} and employ linear-time pathwise sampling of this sparse GP together with leapfrog integration \cite{hairer06}.

\subsection*{Wind interpolation experiment}
In this experiment, the base manifold is the sphere $\bb{S}^2$, which we embed naturally in $\R^3$ as
\[
\f{emb}(\phi, \theta) = (\cos{\theta} \sin{\phi}, \,\sin{\theta} \sin{\phi}, \,\cos{\phi}),
\]
where we used spherical coordinates $\phi \in (0, \pi), \theta \in [0, 2 \pi)$ to parametrise the sphere 
\[
(\phi, \theta) \in \{(0,0)\} \cup \{(\pi, 0)\} \cup (0, \pi) \x  [0, 2 \pi)
\]
We choose a frame $F = (e_1, e_2)$, where $e_1(\phi, \theta) = \hat{\phi}$ and $e_2(\phi, \theta) = \hat{\theta}$ are the unit vectors in the $\phi, \theta$ directions respectively for all $\phi \in (0, \pi), \theta \in (0, 2 \pi)$. The choice of points on the North and South poles determines the choice of gauge at these points. The corresponding orthonormal projection matrix reads
\[\label{sphere-projection}
\m{P}_{\phi, \theta} =
\begin{bmatrix}
\cos{\theta} \cos{\phi} & \sin{\theta}\cos{\phi}  & -\sin{\phi} \\
-\sin{\theta}  & \cos{\theta} & 0
\end{bmatrix},
\]
for all points, with the choice of $\theta=0$ giving the choice of frame at the poles.

For the data, we used the following publicly available data sets.
\1* The ERA5 atmospheric reanalysis data. In particular, the variables \textsc{10m-u-component-of-wind} and \textsc{10m-v-component-of-wind} from the \textsc{reanalysis-era5-single-levels} dataset for the date 01/01/2019 09:00-10:00, regridded from 0.25$^\circ$ to 5.625$^\circ$ resolution using python's \textsc{xESMF} package.
\2* The WeatherBench dataset \cite{rasp2020weatherbench}, which can be found at \url{https://github.com/pangeo-data/WeatherBench}. In particular the variables \textsc{10m-u-component-of-wind} and \textsc{10m-v-component-of-wind} at 5.625$^\circ$ resolution for the entire available period 1979/01/01 - 2018/12/31.
\3* The Aeolus trajectory data, which can be read using Python's \textsc{skyfield} API from Aeolus' two-line element set given below. \\
    \textsc{1 43600U 18066A   21153.73585495  .00031128  00000-0  12124-3 0  9990} \\
    \textsc{2 43600  96.7150 160.8035 0006915  90.4181 269.7884 15.87015039160910}
\0*
Instead of using actual observations from the Aeolus satellite, we generated our own by interpolating the ERA5 data along the satellite track, whose locations are available minutely. This is so that we can compare the predictions against the ground truth to assess the performance.
We use one hour of data, and hence 60 data points, to perform a spatial interpolation instead of a space-time interpolation, which is reasonable as the atmosphere hardly moves during that time period at the spatial scale of interest. Moreover, we include the weekly climatology as prior information (computed by taking the temporal average of historical global wind patterns for each of the 52 calendar weeks during the period 1979-2018 in WeatherBench), which captures general circulation patterns such as trade winds in the poles and the equator. This is equivalent to training the GP on the difference of the wind velocity from the weekly climatology.

For the kernel, we used Mat\'ern-3/2 on the sphere and the Euclidean space (see \textcite{Borovitskiy2020} for the construction of Mat\'ern kernels on the sphere), where the prior amplitude parameter was set to a fixed value (11.5 in the spherical case and 2.2 in the Euclidean case) and the length scale parameter was learnt from data. 
We have tried to learn the length scale initially by fitting the GP on the satellite observations and maximizing the marginal likelihood. 
However, this gave an unrealistically small value, likely due to the observations being too sparse: so, instead, we first trained a sparse GP on 150 randomly chosen time slices of the weatherbench historical wind reanalysis data and minimizing the Kullback--Leibler divergence of the variational distribution from the posterior (using the Adam optimizer with learning rate \texttt{1e-2}). The mean of the learnt length scales of the 150 samples was then used as the final length scale. Denoting by $k_{\bb{S}^2}$ the scalar Mat\'ern-3/2 kernel on the sphere, we construct a matrix-valued kernel on the ambient space $\R^3$ by taking $\v\kappa = k_{\bb{S}^2} \m{I}_{3 \x 3}$ as in the dynamics experiment, which is then used to construct the projected kernel with the projection given by \eqref{sphere-projection}.
Finally, we note that when fitting the GP on the satellite observations, we use an observation error of 1.7m/s, which reflects the sum of the random and systematic error in the Aeolus satellite, as detailed by its technical specifications \cite{aeolus-info}.

\end{document}